%% file: calibsft_arxiv.tex
\documentclass{article}
\usepackage{iclr2027_conference,times}

\input{math_commands.tex}

\usepackage{url}
\usepackage{microtype}
\usepackage{graphicx}
\usepackage{booktabs}
\usepackage{pdfpages}
\usepackage{caption}
\usepackage{tabularx}
\usepackage{subcaption}
\usepackage{booktabs}
\usepackage{longtable}
\usepackage{pdflscape}
\usepackage{algorithm}
\usepackage{algorithmicx}
\usepackage{algpseudocode}
\usepackage{amsmath}
\usepackage{amssymb}
\usepackage{mathtools}
\usepackage{amsthm}
\usepackage{color}
\usepackage{amsfonts}
\usepackage{booktabs}
\usepackage{xtab}
\usepackage{afterpage}
\usepackage{booktabs}
\usepackage{multirow}
\usepackage{caption}
\usepackage{wrapfig}
\usepackage{enumitem}
\usepackage{subcaption}
\IfFileExists{colortbl.sty}{\usepackage{colortbl}}{\newcommand{\rowcolor}[1]{}}
\usepackage[T1]{fontenc}
\usepackage[hypcap=false]{caption}

\title{On the Pitfalls of Verbalized Confidence Priors for Calibrating Large Reasoning Models}

\author{
  Shuoyuan Wang\textsuperscript{1},\enspace
  Beier Luo\textsuperscript{2},\enspace
  Hao Zeng\textsuperscript{3,1},\enspace
  Chengyao Yu\textsuperscript{1},\enspace
  Songxin Zhang\textsuperscript{1},\\
  \textbf{Zejian Xie\textsuperscript{1},\enspace
  Bingyi Jing\textsuperscript{4,5},\enspace
  Hongxin Wei\textsuperscript{1}}\thanks{Correspond to \texttt{weihx@sustech.edu.cn}.} \\
  \textsuperscript{1}Southern University of Science and Technology\quad
  \textsuperscript{2}Nanyang Technological University \\
  \textsuperscript{3}University of Electronic Science and Technology of China \\
  \textsuperscript{4}The Chinese University of Hong Kong, Shenzhen\quad
  \textsuperscript{5}Shenzhen Loop Area Institute
}

\iclrfinalcopy

\definecolor{mydarkred}{rgb}{0.6,0,0}
\definecolor{mydarkgreen}{rgb}{0,0.6,0}
\definecolor{mydarkblue}{rgb}{0.1,0.2,0.55}
\definecolor{calibrow}{gray}{0.94}

\newtheorem{proposition}{Proposition}

\renewcommand\arraystretch{1.2}

\usepackage[colorlinks,
linkcolor=mydarkred,
citecolor=mydarkgreen,
urlcolor=mydarkblue]{hyperref}

\begin{document}

\maketitle

\begin{abstract}
Large reasoning models (LRMs) often suffer from overconfidence when expressing their uncertainty.
Confidence-aware reinforcement learning (RL) offers a promising way to optimize calibration.
However, it relies on on-policy rollouts and is thus constrained by the model's pre-RL confidence distribution, which we term \emph{confidence prior}.
In this work, we reveal that off-the-shelf LRMs exhibit a confidence prior heavily concentrated on a few high values, which persists throughout RL.
Theoretically, we prove that this concentration suppresses policy gradient updates for rarely sampled confidence values and inflates the lower bound on expected Brier risk.
To overcome this exploration bottleneck, we propose \textbf{CalibSFT}, a plug-and-play supervised fine-tuning stage that shapes a calibrated confidence prior with broad support before RL.
For each question, CalibSFT constructs confidence targets combining its success rate with response-level correctness, which provably preserves proper-scoring optimality, and then balances training responses across the confidence spectrum to enable diverse confidence exploration during RL.
To learn from incorrect responses without imitating their reasoning, CalibSFT introduces correctness-conditional supervision, guiding confidence across all responses while supervising reasoning only on correct ones.
Across 16 mathematical and general reasoning benchmarks, incorporating CalibSFT reduces calibration errors and improves discrimination across five representative RL algorithms while preserving comparable accuracy.
Furthermore, CalibSFT delivers practical benefits for downstream selective prediction and model routing.
Our code is available at \url{https://github.com/ml-stat-Sustech/verbalized-confidence-training}.
\end{abstract}

\addtocontents{toc}{\protect\setcounter{tocdepth}{-10}}

\section{Introduction}

Large reasoning models (LRMs) have made substantial progress on complex tasks such as mathematics and coding~\citep{guo2025deepseek,yang2025qwen3}.
However, these accuracy gains do not imply reliable confidence: LRMs often remain overconfident in incorrect answers~\citep{mei2025reasoning}.
A reliable model should express \emph{well-calibrated} confidence~\citep{guo2017calibration}, matching its probability of correctness.
Calibrated confidence indicates when an answer can be trusted, which enables practical applications such as selective prediction~\citep{shen2026provable} and model routing~\citep{hao2026racer}.

Beyond estimates derived from token probabilities, internal representations, or sampling consistency~\citep{luo2025your,srey2026signals,kuhn2023semantic},
models can directly verbalize confidence alongside their answers.
Unlike these alternatives, verbalized confidence can be obtained in a single black-box generation.
However, off-the-shelf LRMs typically produce miscalibrated scores that concentrate on a few high values regardless of correctness~\citep{mei2025reasoning,cacioli2026saturation}.
To address this issue, existing post-training approaches calibrate verbalized confidence via either supervised fine-tuning (SFT) on constructed targets~\citep{xu2024sayself,zhang2026lovec} or reinforcement learning (RL) guided by calibration rewards~\citep{damani2026beyond,ma2026decoupling}.

Among these approaches, confidence-aware RL is particularly promising owing to its superior generalization beyond the training distribution~\citep{chu2025sft}.
However, because confidence-aware RL explores via on-policy rollouts, its optimization is fundamentally constrained by the confidence distribution before RL, which we term \emph{confidence prior}.
In practice, this prior concentrates heavily on a narrow set of high values across both correct and incorrect responses (Figure~\ref{fig:initial_confidence_prior}).
Although calibration rewards explicitly penalize miscalibration, this concentration persists throughout policy optimization and limits the exploration of diverse confidence values (Figure~\ref{fig:rlcr_training_calibration}).
Our theoretical analysis in Section~\ref{sec:motivation} reveals that the policy gradient for each confidence value scales with its sampling probability, which suppresses updates on rarely sampled values.
The analysis further shows that a concentrated prior induces miscalibration by raising the lower bound on expected Brier risk.
As a result, post-RL confidence fails to reflect task difficulty in practice: mean confidence on MinervaMath is only 9.7\% lower than on Math500 despite a 39.1\% drop in Pass@1 (Figure~\ref{fig:rlcr_test_calibration}).

To overcome this exploration bottleneck, we propose \textbf{CalibSFT}, a supervised fine-tuning stage that shapes a calibrated confidence prior with broad support before confidence-aware RL.
Specifically, for each training question, CalibSFT samples multiple responses from the base model and assigns each a confidence target that balances the question-level success rate with response-level correctness.
We further prove that this mixed target preserves the question-level optimal confidence under any strictly proper scoring rule.
To enable diverse confidence exploration during RL, it employs target-balanced sampling, drawing an equal number of responses for each target value.
During fine-tuning, it introduces correctness-conditional supervision, guiding confidence on all responses while restricting reasoning and answer supervision to correct ones, so that the model learns low confidence on incorrect responses without imitating their solutions.
Since CalibSFT reshapes the confidence prior without altering RL objectives, it can be seamlessly integrated with any confidence-aware RL framework.

We evaluate CalibSFT on 16 benchmarks spanning both in-distribution mathematical reasoning and out-of-distribution general reasoning of varying difficulty.
Across 5 confidence-aware RL methods, initializing with CalibSFT outperforms base-model initialization in both calibration and discrimination without compromising task accuracy (Table~\ref{tab:main_results}).
Consistent with our analysis, CalibSFT mitigates confidence concentration after RL (Figure~\ref{fig:confidence_distribution}) and strengthens the correlation between task-level confidence and accuracy (Table~\ref{tab:confidence_correlation}).
It also remains effective across different confidence formats and model families (Tables~\ref{tab:format_calibsft} and~\ref{tab:gemma_calibsft}).
Furthermore, CalibSFT provides practical benefits for downstream applications like selective prediction and model routing (Appendix~\ref{sec:selective_risk_control}--\ref{sec:model_routing}).

Our contributions are summarized as follows:
\begin{enumerate}[leftmargin=2em,itemsep=1pt,topsep=1pt]
    \item We identify that LRMs exhibit a concentrated confidence prior that limits on-policy RL exploration, and theoretically prove that this concentration suppresses policy gradients and hinders calibration by increasing the lower bound on expected Brier risk.
    \item We propose CalibSFT, a supervised fine-tuning stage that constructs proper-scoring confidence targets and balances them across the confidence spectrum to enable RL exploration.
    \item Extensive evaluations across 16 benchmarks demonstrate that CalibSFT delivers substantial calibration and discrimination gains across all five RL baselines, while exhibiting remarkable robustness across diverse confidence formats, architectures, and downstream applications.
\end{enumerate}

\section{Preliminaries}
In this work, we study the calibration of verbalized numerical confidence in large reasoning language models (LRMs)~\citep{xiong2024can,damani2026beyond}.
Given a question $x$, an LRM $\pi_\theta$ generates a response $o = (r, a, c)$, where $r$ is a reasoning trace, $a$ is an answer, and $c \in [0, 1]$ is the model's estimated probability that $a$ is correct.
The model verbalizes $c$ conditioned on the prefix $(x, r, a)$ via $\pi_\theta(c \mid x, r, a)$.
An automated verifier $V$ evaluates the answer against the ground truth to provide a correctness label $z = V(x, a) \in \{0, 1\}$.

\subsection{Calibration Metrics}

To evaluate verbalized confidence $c$, we assess both calibration and discrimination.
Let $C\in[0,1]$ and $Z\in\{0,1\}$ denote confidence and correctness random variables, respectively.
Confidence is perfectly calibrated if $\Pr(Z=1\mid C)=C$ almost surely, and discriminative if it ranks correct responses above incorrect ones.
We measure calibration via \textbf{Expected Calibration Error (ECE)}~\citep{guo2017calibration}, defined as $\mathbb E_C[|\Pr(Z=1\mid C)-C|]$.
With $N$ responses in $B$ bins $\{\mathcal I_b\}_{b=1}^{B}$, ECE is estimated as:

\begin{equation}
    \operatorname{ECE}
    =\sum_{b=1}^{B}\frac{|\mathcal I_b|}{N}
    \left|\operatorname{acc}(\mathcal I_b)-
    \operatorname{conf}(\mathcal I_b)\right|,
\end{equation}
where $\operatorname{acc}(\mathcal I_b)$ and $\operatorname{conf}(\mathcal I_b)$ are the mean correctness and confidence in bin $b$.
While ECE assesses bin-level calibration, the \textbf{Brier score}~\citep{brier1950verification} computes response-level squared error,
$\frac{1}{N}\sum_{j=1}^{N}(c_j-z_j)^2$.
However, calibration does not guarantee discrimination. For example,
assigning constant confidence equal to overall accuracy yields zero ECE while giving correct and incorrect responses
identical scores~\citep{eusebi2026reading}.
We therefore report \textbf{AUROC} to evaluate threshold-agnostic discrimination between correct and incorrect answers.

\subsection{Confidence-aware Reinforcement Learning}

\paragraph{Confidence-aware rewards.}
Reinforcement learning with verifiable rewards (RLVR) trains the policy using the correctness label $z$ as the reward.
To optimize calibration alongside accuracy, confidence-aware RL introduces an additional calibration reward depending on both $c$ and $z$.
For example, RLCR~\citep{damani2026beyond} adopts the negative Brier loss, yielding the combined reward:
\begin{equation}
    R(o)=z-(c-z)^2.
\end{equation}

\paragraph{Policy optimization.}
Confidence-aware RL typically optimizes this reward with Group Relative
Policy Optimization (GRPO)~\citep{shao2024deepseekmath}. For each
question, GRPO samples $G$ responses $\{o_i\}_{i=1}^{G}$ from $\pi_{\theta_{\mathrm{old}}}$ and normalizes the group-wise rewards
$R_i=R(o_i)$ to obtain the advantage:
\begin{equation}
    \widehat A_i=\frac{R_i-\bar R}{\sigma_R+\epsilon},\qquad
    \bar R=\frac{1}{G}\sum_{j=1}^{G}R_j,\qquad
    \sigma_R=\operatorname{std}\!\left(\{R_j\}_{j=1}^{G}\right),
\end{equation}
where $\epsilon>0$ ensures numerical stability. GRPO then updates the
policy by maximizing
\begin{equation}
J_{\mathrm{GRPO}}(\theta)=\mathbb E\Bigg[\frac{1}{G}\sum_{i=1}^{G}
\frac{1}{L_i}\sum_{t=1}^{L_i}\min\Big(\rho_{i,t}(\theta)\widehat A_i,\,
\operatorname{clip}\!\left(\rho_{i,t}(\theta),1-\epsilon_{\mathrm{clip}},
1+\epsilon_{\mathrm{clip}}\right)\widehat A_i\Big)\Bigg],
\end{equation}
where $o_{i,t}$ is the $t$-th token of $o_i$, $L_i$ is the length of
$o_i$, $\rho_{i,t}(\theta)=\pi_\theta(o_{i,t}\mid x,o_{i,<t})/
\pi_{\theta_{\mathrm{old}}}(o_{i,t}\mid x,o_{i,<t})$ is the importance
ratio, and $\epsilon_{\mathrm{clip}}$ is the clipping parameter.
Because this optimization is on-policy, it only learns from confidence values sampled during rollouts, which motivates us to examine the confidence distribution of LRMs before and during RL.

\section{Motivation}
\label{sec:motivation}

While confidence-aware RL penalizes overconfidence
through explicit reward formulations~\citep{damani2026beyond,ma2026decoupling,yang2026scaling},
its efficacy remains constrained by the support of sampled confidence values.
To understand this limitation, we examine the confidence prior before RL, its dynamics during RL, and calibration at test time.

\begin{figure}[t]
    \centering
    \includegraphics[width=\linewidth]{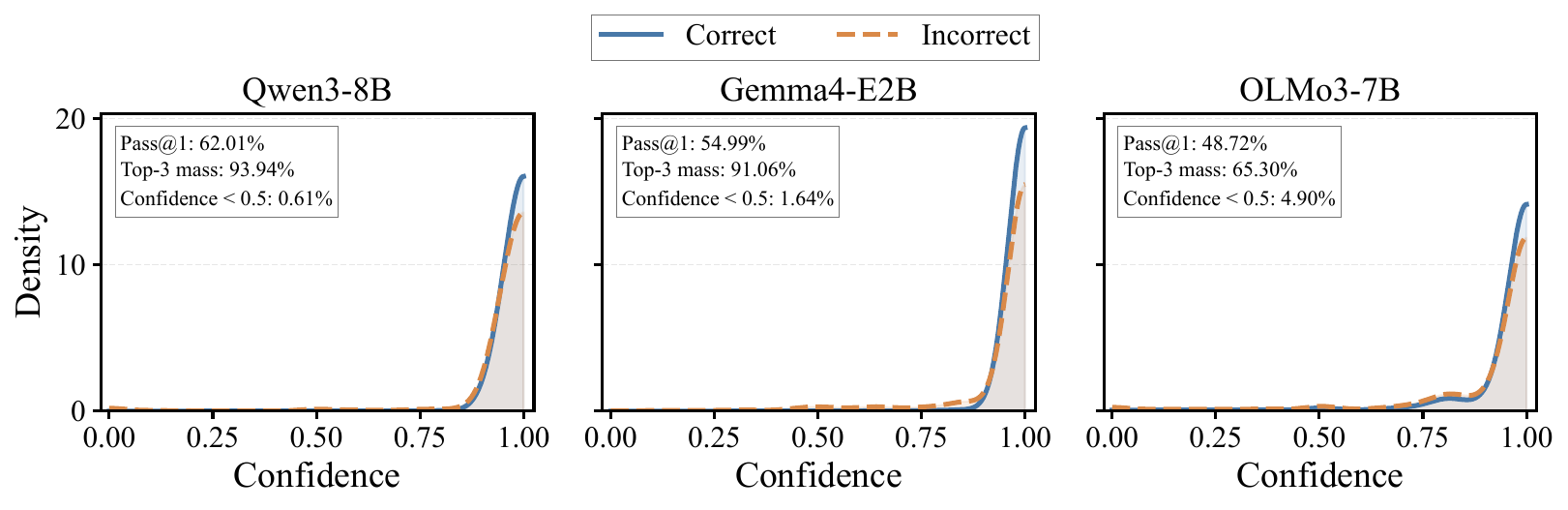}
    \caption{Confidence distributions from three base models on DeepScaleR-Train.
    Confidence concentrates on a few high values for both correct and
    incorrect responses.}
    \label{fig:initial_confidence_prior}
\end{figure}

\paragraph{Confidence priors concentrate on a few high values.}
We refer to the confidence distribution before RL as the \emph{confidence prior}.
To characterize this prior, we sample 10 responses per question on DeepScaleR-Train using
Qwen3-8B~\citep{yang2025qwen3}, Gemma4-E2B~\citep{team2026gemma},
and OLMo3-7B~\citep{olmo2025olmo3}.
As shown in Figure~\ref{fig:initial_confidence_prior},
all models place most of their confidence mass on a few high values, with fewer than 5\% of responses below 0.5.
This concentration at high confidence is also observed among incorrect answers.
Consequently, RL starts from a concentrated confidence prior and rarely samples low confidence values early in training.

\paragraph{RL improves calibration but confidence concentration persists.}
To examine whether confidence-aware RL overcomes this concentration, we train Qwen3-8B with RLCR~\citep{damani2026beyond} on DeepScaleR-Train.
As shown in Figure~\ref{fig:rlcr_training_calibration}(a), while the calibration reward lowers ECE, confidence discrimination (AUROC) fails to increase.
Figure~\ref{fig:rlcr_training_calibration}(b) further explains this failure: the number of distinct confidence values sampled per question closely tracks the confidence token entropy,
which remains low throughout training and limits exploration across the confidence space.
Consequently, the model merely shifts probability mass among a few dominant values to satisfy average calibration, without learning to distinguish correct from incorrect responses.
In Appendix~\ref{sec:training_dynamics}, we show that this restricted exploration and low confidence entropy persist across other confidence-aware RL methods.

To understand why this concentration persists,
we formulate verbalized confidence generation as action selection under a softmax policy~\citep{garg2022alternate}.
Let $h = \phi(x, r, a)$ denote the context representation on which the policy conditions to verbalize confidence, so that $\pi_\theta(c \mid x, r, a) = \pi_\theta(c \mid h)$, and let $p_h = \Pr(Z = 1 \mid h)$.
Since $Z = V(x, a)$ is determined by the prefix and $C$ depends on the prefix only through $h$, we have $C \perp Z \mid h$.

\begin{proposition}[Policy gradient suppression under concentrated priors]
\label{prop:confidence_saturation}
Given a context representation $h$ and a finite set of confidence values $\{c_1, \ldots, c_K\} \subseteq [0, 1]$, let $\eta = (\eta_1, \ldots, \eta_K)$ denote the logits with $\pi_i = \Pr(C = c_i \mid h) = \exp(\eta_i)/\sum_{j=1}^K\exp(\eta_j)$.
Define the expected reward $u_i(h) = -\mathbb E[(c_i - Z)^2 \mid h]$.
Let $J_h(\eta) = \sum_{i=1}^K \pi_i u_i(h)$ and $\Delta_h = \max_i u_i(h) - \min_i u_i(h) \le 1$. Then
\begin{equation}
    \label{eq:logit_gradient}
    \frac{\partial J_h}{\partial\eta_i} = \pi_i \bigl(u_i(h) - J_h\bigr),
    \qquad
    \left|\frac{\partial J_h}{\partial\eta_i}\right| \le \pi_i \Delta_h.
\end{equation}
If $\pi_k = 1 - \varepsilon$ for some dominant action $k \in \{1, \ldots, K\}$, then
\begin{equation}
    \label{eq:logit_gradient_bound}
    \|\nabla_\eta J_h\|_1 \le 2\varepsilon \Delta_h \le 2\varepsilon.
\end{equation}
\end{proposition}

Proposition~\ref{prop:confidence_saturation} shows that the logit gradient for any action $c_i$ scales with its sampling probability $\pi_i$, consistent with the $\mathcal{O}(\pi_i)$ gradient decay in softmax policies~\citep{mei2020global}.
Under a concentrated prior ($\pi_k = 1 - \varepsilon$), the total gradient norm is bounded by $2\varepsilon\Delta_h$.
Crucially, when the dominant value $c_k$ is far from $p_h$, the gradient that shifts probability mass away from $c_k$ is also bounded by $\varepsilon\Delta_h$, which hinders on-policy RL from escaping this suboptimal initialization and explains why it struggles to correct overconfidence even under calibration rewards.
We provide the proof in Appendix~\ref{sec:confidence_saturation_proof}.

\paragraph{Persistent confidence concentration hinders downstream calibration.}
Beyond restricting training exploration, this concentrated prior misaligns confidence with task difficulty at test time.
As shown in Figure~\ref{fig:rlcr_test_calibration}, across two benchmarks of contrasting difficulty (Math500 and MinervaMath), Pass@1 drops sharply by 39.1\%, yet mean confidence changes by only 9.7\%.
Because the RL-trained model still concentrates its confidence on a few high values, it struggles to assign lower confidence to hard questions.
To theoretically characterize the cost of this miscalibration,
we analyze the expected Brier loss $\mathbb E[(C-Z)^2 \mid h]$~\citep{gneiting2007strictly} and establish the following lower bound:

\begin{proposition}[Brier risk under prior misalignment]
\label{prop:confidence_support}
Suppose $C \perp Z \mid h$. For any $\delta > 0$, let $m_\delta(h) = \Pr(|C - p_h| < \delta \mid h)$. Then
\begin{equation}
    \label{eq:brier_lower_bound}
    \mathbb E[(C-Z)^2 \mid h] \ge p_h(1-p_h) + \delta^2 \bigl(1-m_\delta(h)\bigr).
\end{equation}
\end{proposition}

Proposition~\ref{prop:confidence_support} lower-bounds the expected Brier loss
by the conditional variance $p_h(1-p_h)$ and an excess error $\delta^2\bigl(1-m_\delta(h)\bigr)$ from confidence deviations.
When the policy rarely samples values near $p_h$ (i.e., $m_\delta(h)\approx 0$), this second term approaches $\delta^2$.
Because confidence-aware RL only receives reward feedback on sampled values and suppresses gradients on rarely sampled values (Proposition~\ref{prop:confidence_saturation}), calibration rewards struggle to allocate mass to values near $p_h$ that the policy rarely explores.
In Appendix~\ref{sec:confidence_support_proof}, we provide the full proof and extend this bound to best-of-$G$ sampling.

\begin{figure}[t]
    \centering
    \begin{minipage}[t]{0.66\linewidth}
        \centering
        \includegraphics[width=\linewidth]{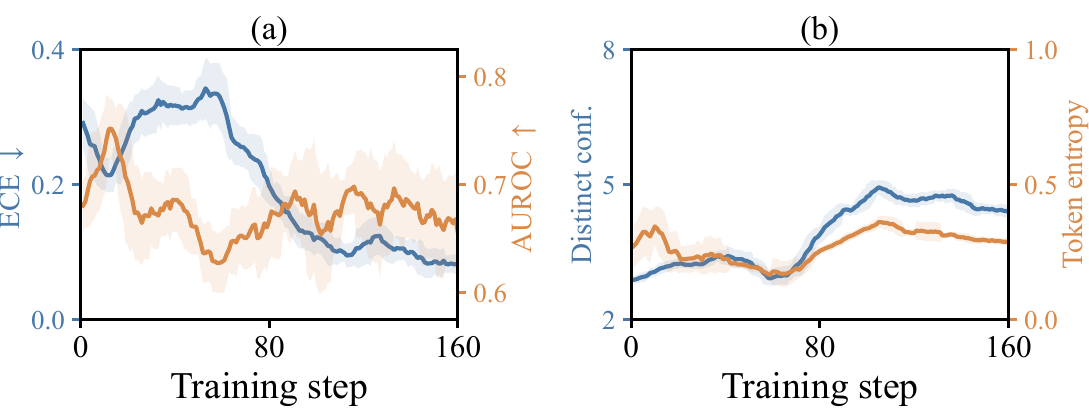}
        \caption{Training dynamics of RLCR on Qwen3-8B.
        (a) The calibration reward lowers ECE, but AUROC does not improve.
        (b) The number of distinct confidence values sampled per question closely tracks confidence token entropy, and their persistently low levels prevent sufficient exploration across the confidence space.}
        \label{fig:rlcr_training_calibration}
    \end{minipage}\hfill
    \begin{minipage}[t]{0.32\linewidth}
        \centering
        \includegraphics[width=\linewidth]{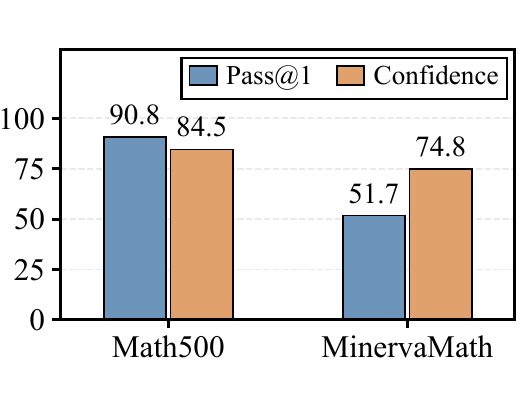}
        \caption{Pass@1 and mean confidence on Math500 and MinervaMath of
        RLCR on Qwen3-8B. Mean confidence shifts far less than Pass@1 across tasks.}
        \label{fig:rlcr_test_calibration}
    \end{minipage}
\end{figure}

\section{Method}

Section~\ref{sec:motivation} shows that confidence concentration persists throughout RL, since rarely sampled confidence values receive minimal policy gradients.
Instead of relying on on-policy exploration, we propose CalibSFT to shape a well-calibrated confidence prior before policy optimization.
CalibSFT first constructs confidence targets that mix question-level success rates with response-level correctness, preserving each question's optimal confidence while supporting within-question discrimination.
It then applies target-balanced sampling to broaden support across the confidence spectrum, so that RL can explore diverse confidence values.
Finally, it supervises confidence on all responses but restricts reasoning and answer supervision to correct ones, so that the model learns to express low confidence on errors without reinforcing incorrect solutions.

\subsection{Proper-Scoring Confidence Targets}
Under strictly proper scoring rules, optimal confidence matches the true correctness probability~\citep{gneiting2007strictly}.
As this probability is not observed directly, we estimate it from base-model rollouts.
For each question $x$, we sample $n$ responses with correctness $z_i \in \{0, 1\}$ and success rate $q_x = \frac{1}{n} \sum_{i=1}^n z_i$.
However, $q_x$ alone assigns identical targets to all responses and removes within-question discrimination.
To balance difficulty calibration with response discrimination, we construct the confidence target for response $i$ as:

\begin{equation}
    \label{eq:calibsft_target}
    c_i^* = \lambda q_x + (1 - \lambda) z_i,
\end{equation}
where the mixing weight $\lambda\in[0,1]$ balances question-level success rate with response correctness.
We further show that $c_i^*$ preserves the question-level optimal confidence under any strictly proper scoring rule.

\begin{proposition}[Proper-scoring optimality of the mixed target]
\label{prop:confidence_target_moments}
Let $\ell(\hat c, z)$ be a strictly proper scoring loss for a prediction $\hat c \in [0, 1]$ and $z \in \{0, 1\}$, extended to $t \in [0, 1]$ via $\ell(\hat c, t) = t \, \ell(\hat c, 1) + (1 - t) \, \ell(\hat c, 0)$.
For a question $x$, let $Z_1, \ldots, Z_n$ denote the correctness of $n$ independent base-model rollouts with $p_x = \Pr(Z_i = 1 \mid x)$, and let $C_i^* = \lambda Q_x + (1 - \lambda) Z_i$ with $Q_x = \frac{1}{n} \sum_{j=1}^n Z_j$.
Then $\mathbb E[C_i^* \mid x] = p_x$, and for any $\hat c$ at which $\ell$ is finite,
\begin{equation}
    \label{eq:proper_scoring_equivalence}
    \mathbb E[\ell(\hat c, C_i^*) \mid x] = \mathbb E[\ell(\hat c, Z_i) \mid x].
\end{equation}
Consequently, $p_x$ remains the unique minimizer of the expected scoring loss on question $x$.
\end{proposition}

Proposition~\ref{prop:confidence_target_moments} shows that, under the rollout distribution, the mixed target preserves the optimal confidence $p_x$ for each question.
A larger $n$ reduces target variance, while $\lambda$ sets a target margin of $1-\lambda$ between correct and incorrect responses within a question.
We use $n=50$ and $\lambda=0.5$ by default, and provide the proof, variance derivations, and conditioning analysis in Appendix~\ref{sec:confidence_target_proof}.

In practice, we replace each response's confidence with $c_i^*$, preserving its reasoning trace and answer.
Because many questions are solved by all or none of their rollouts, 61\% of raw targets are exactly 0 or 1, so training directly on this pool would concentrate the SFT prior at these extremes.
We therefore apply \emph{target-balanced sampling}, which draws an equal number of responses for each target value (Appendix~\ref{sec:training_pipelines}), establishing a well-supported initial prior before downstream RL.
Since balancing reweights responses away from the rollout distribution in Proposition~\ref{prop:confidence_target_moments}, we verify in Table~\ref{tab:balance_ablation} that it maintains comparable calibration while broadening confidence support.

\subsection{Correctness-Conditional Supervision}

The constructed data contain both correct and incorrect responses. The
incorrect ones are kept because their lower targets guide the model to
express low confidence on its own errors, whereas imitating their
reasoning would train the model on incorrect solutions.
CalibSFT therefore supervises confidence on all responses and restricts
reasoning and answer supervision to correct ones.

Formally, let $\widetilde o$ denote a response whose confidence has been
replaced by its constructed target, and let $\mathcal T_{r,a}$
and $\mathcal T_c$ denote the positions of its reasoning/answer tokens
and its confidence tokens, respectively.
The CalibSFT objective is then
\begin{equation}
\begin{split}
    \mathcal L_{\mathrm{CalibSFT}}={}&(1-\beta)\,\mathbb E_{\mathrm{correct}}
    \left[-\frac{1}{|\mathcal T_{r,a}|}\sum_{t\in\mathcal T_{r,a}}
    \log\pi_\theta(\widetilde o_t\mid x,\widetilde o_{<t})\right]\\
    &+\beta\,\mathbb E_{\mathrm{all}}
    \left[-\frac{1}{|\mathcal T_c|}\sum_{t\in\mathcal T_c}
    \log\pi_\theta(\widetilde o_t\mid x,\widetilde o_{<t})\right].
\end{split}
\end{equation}
Here $\mathbb E_{\mathrm{correct}}$ averages over the correct responses
and $\mathbb E_{\mathrm{all}}$ over all of them.
Averaging within each segment prevents the longer reasoning sequences
from dominating the loss by token count.
The weight $\beta$ balances
confidence supervision against reasoning supervision, and we set
$\beta=0.5$ so that the two terms contribute equally.

\paragraph{Integration with confidence-aware RL.}
CalibSFT provides an initialization for confidence-aware RL
without changing its objective, and thus applies to different RL methods.
We evaluate CalibSFT across RL methods in Section~\ref{sec:overall_results},
analyze its impact on confidence distributions in Section~\ref{sec:calibration_analysis},
and test its generality across confidence formats and base models in Section~\ref{sec:generality}.
Finally, Appendix~\ref{sec:ablation_studies} ablates the target construction, mixing weight $\lambda$, and supervision design.

\section{Experiments}

\subsection{Experimental Setup}

\paragraph{Baselines.}
We compare methods spanning answer-only and confidence-aware training. \textbf{Base}
is the off-the-shelf checkpoint, while \textbf{RLVR} optimizes only answer correctness.
We evaluate five representative confidence-aware RL baselines covering two training paradigms:
(1)~\textbf{joint training}, which simultaneously optimizes correctness and confidence, including \textbf{RLCR}~\citep{damani2026beyond}, \textbf{CoCA}~\citep{li2026confidence}, \textbf{DCPO}~\citep{ma2026decoupling}, and \textbf{C3RL}~\citep{yang2026scaling}.
(2)~\textbf{two-stage training}, where \textbf{ReDoubt}~\citep{bani2026rewarding} freezes reasoning optimization and tunes confidence on top of a trained RLVR model.
For each method, we compare initialization from the base model and from CalibSFT.
All methods use Qwen3-8B~\citep{yang2025qwen3} as the default backbone within the \textbf{DAPO} training framework~\citep{yu2025dapo}.

\paragraph{Datasets.}
We use \textbf{DeepScaleR-Preview}~\citep{deepscaler2025} as the mathematical training dataset.
We evaluate on a broad range of in-distribution (ID) and out-of-distribution (OOD)
datasets spanning different levels of difficulty.
For ID mathematical reasoning, we use \textbf{DeepScaleR-Eval},
\textbf{MATH-500}~\citep{hendrycks2021math}, \textbf{MinervaMath}~\citep{lewkowycz2022minerva},
\textbf{OlympiadBench}~\citep{he2024olympiadbench}, \textbf{GSM8K}~\citep{cobbe2021gsm8k}, and
\textbf{AIME 2024--2026}~\citep{dekoninck2026matharena}.
For general reasoning, we use \textbf{HotpotQA}~\citep{yang2018hotpotqa},
\textbf{MuSiQue}~\citep{trivedi2022musique}, \textbf{TriviaQA}~\citep{joshi2017triviaqa},
\textbf{NQOpen}~\citep{lee2019latent}, \textbf{PopQA}~\citep{mallen2023trust},
\textbf{WebQuestions}~\citep{berant2013semantic}, \textbf{DROP}~\citep{dua2019drop}, and
\textbf{LiveBenchReasoning}~\citep{white2025livebench}.

\paragraph{Evaluation metrics.}
We evaluate model performance using complementary metrics that capture answer
correctness, confidence discrimination, and absolute calibration. We report \textbf{Pass@1} for answer
correctness, \textbf{AUROC} for ranking correct responses above incorrect responses, and
\textbf{Brier score}~\citep{brier1950verification} and
\textbf{ECE}~\citep{guo2017calibration} for probabilistic calibration.
We present additional metrics for calibration and confidence diversity in Appendix~\ref{sec:additional_metrics}.

\paragraph{Implementation details.}
During evaluation, we use temperature 0.6 and sample 32 responses per problem
for the AIME benchmarks and 4 responses per problem for all other benchmarks. All experiments are conducted on a
single node with eight NVIDIA B20Z GPUs.  We report our training details in Appendix~\ref{sec:detailed_setup} and training cost analysis in Appendix~\ref{sec:training_cost}.

\subsection{Overall Results}
\label{sec:overall_results}

\paragraph{CalibSFT consistently benefits diverse confidence-aware RL methods.}
As shown in Table~\ref{tab:main_results}, initializing RL from CalibSFT
reduces average Brier score and ECE and increases AUROC across all five confidence-aware RL
baselines on both mathematical and general-reasoning benchmarks.
As a concrete example, CalibSFT reduces DCPO's ECE from 29.05 to 8.83 while lifting its
AUROC from 63.57 to 76.00 on MinervaMath.
Moreover, these calibration gains come without sacrificing average accuracy. The average Pass@1 improves by
1.20 in-distribution and 1.49 out-of-distribution across the five methods.
We provide detailed per-dataset results in Tables~\ref{tab:appendix_id_results}--\ref{tab:appendix_ood_results} (Appendix~\ref{sec:per_dataset_results})
and verify stability across random seeds in Appendix~\ref{sec:seed_stability}.
Overall, these results establish CalibSFT as a plug-and-play initialization that consistently
strengthens confidence-aware RL for better calibration.

\paragraph{CalibSFT and RL contribute complementary gains.}
To disentangle their contributions, we compare CalibSFT and RLVR individually.
CalibSFT alone substantially mitigates overconfidence, with ECE dropping from 34.36 to 14.23 in-distribution and from 46.94 to 25.49 out-of-distribution.
In contrast, Base and RLVR remain miscalibrated, indicating that answer correctness alone is insufficient to reshape confidence.
However, calibration does not imply better confidence ranking: CalibSFT's OOD AUROC even declines from 65.71 to 61.20, reflecting the limited generalization of supervised training.
Confidence-aware RL fills exactly this gap by directly optimizing discrimination, and combining it with CalibSFT yields the strongest overall performance.

\input{tables/main_results.tex}

\subsection{Calibration Analysis}
\label{sec:calibration_analysis}

\paragraph{CalibSFT aligns confidence with accuracy across tasks.}
A well-calibrated model should express higher confidence on tasks it can reliably solve and lower confidence on harder ones.
To test this, we compute Pearson and Spearman correlations between per-dataset mean confidence and per-dataset accuracy for each RL method on the ID and OOD splits separately.
As shown in Table~\ref{tab:confidence_correlation}, CalibSFT improves Spearman correlation for all RL baselines on both mathematical and general reasoning benchmarks.
For example, DCPO's Spearman correlation rises from 0.619 to 0.976 on mathematical tasks and from 0.548 to 0.643 on general reasoning.
Pearson correlation also improves consistently across all baselines and splits (e.g., from 0.505 to 0.729 for RLCR on OOD).
These results indicate that CalibSFT helps the model assign higher confidence to tasks it is more likely to answer correctly (see Appendix~\ref{sec:confidence_visualization} for reliability diagrams across benchmarks of varying difficulty).

\input{tables/confidence_correlation.tex}

\begin{figure}[t]
    \centering
    \includegraphics[width=\linewidth]{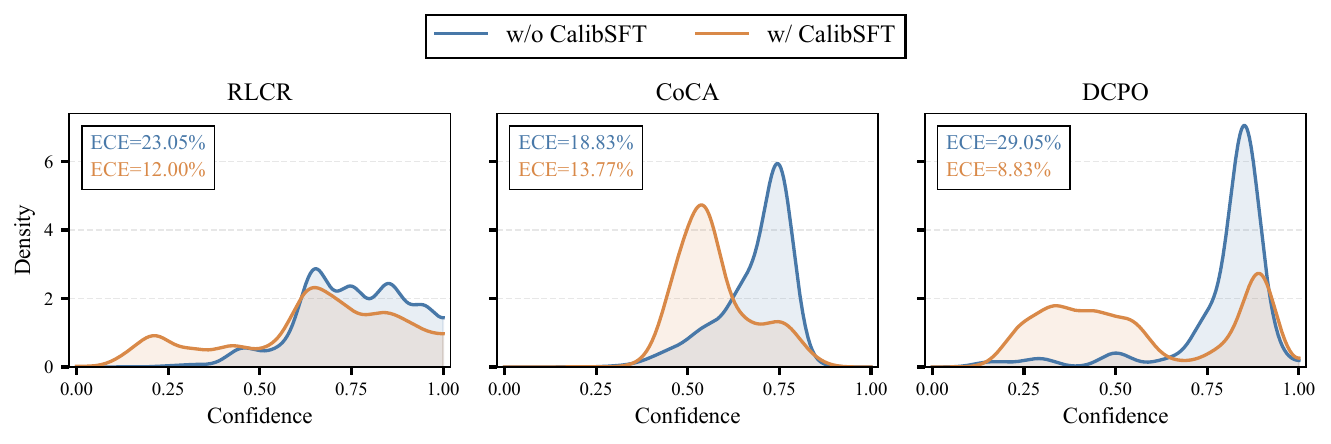}
    \caption{Confidence distributions on MinervaMath for RLCR, CoCA, and DCPO.
    CalibSFT reduces concentration at dominant confidence values and lowers
    ECE for all three methods.}
    \label{fig:confidence_distribution}
\end{figure}

\begin{wrapfigure}{r}{0.58\textwidth}
    \centering
    \vspace{-5pt}
    \includegraphics[width=\linewidth]{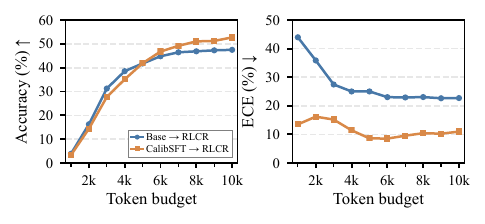}
    \vspace{-15pt}
    \caption{Accuracy ($\uparrow$) and ECE ($\downarrow$) on AIME 2026 across token budgets
    for RLCR. CalibSFT initialization yields lower ECE across all budgets while maintaining accuracy.}
    \label{fig:reasoning_budget}
\end{wrapfigure}

\paragraph{CalibSFT mitigates confidence concentration after RL.}
Prior work observes that confidence estimates often concentrate on a few dominant values even after confidence-aware RL training~\citep{wang2026process,yang2026scaling}.
To examine whether CalibSFT initialization alleviates this concentration, we compare RLCR, CoCA, and DCPO on MinervaMath, with and without CalibSFT.
As shown in Figure~\ref{fig:confidence_distribution}, all three methods trained from Base exhibit heavy peaks at a few high-confidence values, while initializing from CalibSFT flattens these peaks and aligns the mean confidence with the actual accuracy.
For example, DCPO's dominant peak near high confidence is redistributed toward lower values while accuracy is preserved, reducing ECE from 29.05\% to 8.83\%.
These results suggest that CalibSFT initialization propagates lower concentration and better calibration into the downstream RL policy.

\input{tables/format_calibsft.tex}

\input{tables/gemma_calibsft.tex}

\paragraph{The calibration gains of CalibSFT persist across reasoning budgets.}
We compare RLCR with and without CalibSFT on AIME 2026 across a range of test-time token budgets.
As shown in Figure~\ref{fig:reasoning_budget}, CalibSFT achieves substantially lower ECE at every budget while accuracy remains comparable to the RLCR baseline.
For example, at a 10{,}000-token budget, ECE drops from 22.71 to 10.97.

\subsection{Generality of CalibSFT}
\label{sec:generality}

\paragraph{CalibSFT is effective across diverse confidence formats.}
We train RLCR with CalibSFT using four confidence target formats: two-decimal probabilities on a fine grid $\{0.00, 0.01, \ldots, 1.00\}$ or a coarse grid $\{0.00, 0.05, \ldots, 1.00\}$, integer scores from 0 to 9, and linguistic labels (\emph{low}, \emph{medium}, \emph{high}) (see Appendix~\ref{sec:format_details} for details).
As shown in Table~\ref{tab:format_calibsft}, CalibSFT consistently improves both ECE and AUROC across all four formats, indicating that its benefit does not rely on a particular way of expressing confidence.
The reduction in ECE is largest with the fine probability grid, whose granularity lets confidence approximate the probability of correctness most closely.
In Appendix~\ref{sec:confidence_estimators}, we further show that CalibSFT consistently improves calibration under alternative confidence estimators, including internal token probabilities and self-consistency voting.

\paragraph{CalibSFT is agnostic to the base model.}
To examine the transferability of CalibSFT across model architectures, we train Gemma4-E2B~\citep{team2026gemma} using RLCR with CalibSFT, and then evaluate it on three ID and three OOD benchmarks spanning a range of difficulty.
As shown in Table~\ref{tab:gemma_calibsft}, CalibSFT improves both ECE and AUROC on all benchmarks, reducing the average ECE from 23.41 to 12.49 and raising the average AUROC from 70.49 to 73.95.
These results indicate that the calibration benefit of CalibSFT holds across model families on benchmarks of varying difficulty.

\section{Conclusion}

In this work, we address the challenge of verbalized overconfidence in large reasoning models.
We identify that concentrated confidence priors restrict exploration during confidence-aware RL,
and propose CalibSFT to shape a well-calibrated prior with broad support before policy optimization.
CalibSFT derives confidence targets from self-sampled rollouts that provably preserve proper-scoring optimality, and supervises confidence without imitating incorrect reasoning.
Experimental results across 16 benchmarks confirm that CalibSFT improves calibration and discrimination across diverse RL frameworks, establishing a foundation for reliable and trustworthy reasoning models.

\subsection*{AI use statement}
In this work, we used generative AI tools for text polishing, grammar correction, and improving language readability.
All suggested edits were thoroughly reviewed, verified, and finalized by the authors, who take full responsibility for the contents of this paper.

\subsection*{Ethics statement}
This work focuses on understanding and mitigating overconfidence in large reasoning language models to improve their reliability and safety.
All experiments are conducted on publicly available, standard academic benchmarks and open-weight models, and do not involve human subjects, personally identifiable data, or sensitive applications.

\subsection*{Reproducibility statement}
To ensure the reproducibility of our theoretical and empirical findings, complete details are provided throughout the paper and appendix.
For experimental results, the rollout collection process, target-balanced sampling, and full training hyperparameters (e.g., learning rates, batch sizes, and optimization steps) are documented in Appendix~\ref{sec:detailed_setup}.
System prompts for all confidence formats are provided in Table~\ref{tab:format_prompts}, and benchmark grading rules are described in Appendix~\ref{sec:detailed_setup}.
Our code is available at \url{https://github.com/ml-stat-Sustech/verbalized-confidence-training}.

\bibliography{reference}
\bibliographystyle{iclr2027_conference}

\appendix

\addtocontents{toc}{\protect\setcounter{tocdepth}{2}}
\renewcommand{\contentsname}{Appendix}
\clearpage
{\hypersetup{linkcolor=mydarkblue,linktoc=all}\tableofcontents}

\section{Related Work}

\paragraph{Confidence estimation in LLMs.}
Confidence estimates for a model answer can be derived from many signals, of which four
families cover most existing methods and differ in  their model-access requirements
and inference costs~\citep{geng2024survey,liu2025uncertainty}. The most direct signal is the probability
the model assigns to its own answer tokens, used either directly or after recalibration by a layer fitted on labeled
data~\citep{kadavath2022language,liu2024litcab}. Probing methods instead train a classifier
on the hidden states or on the layer-wise trajectory that produces the
answer~\citep{azaria2023internal,eusebi2026reading,srey2026signals}, which requires access
to activations rather than logits. When a model exposes neither, sampling-based methods
recover confidence from behavior alone and measure how far several sampled answers
agree~\citep{kuhn2023semantic,del2026uncertainty,taubenfeld2025confidence}, at the cost of
repeated generations per question. Verbalized confidence removes that cost as well: the
model states a numerical or linguistic score alongside its answer within a single
generation~\citep{lin2022teaching,tian2023just,xiong2024can,yang2024verbalized}. We study
the verbalized setting for reasoning models, where the score follows a reasoning trace and
the final answer. Such a score is informative only once it matches accuracy: the values an
uncalibrated base model reports concentrate on a few points of the
scale~\citep{dai2026rescalingconfidence,cacioli2026saturation} and separate its errors less well
than its own token probabilities do~\citep{ni2024honest,tao2025revisiting}.

\paragraph{Confidence calibration for LLMs.}
Calibration requires the reported confidence to match the empirical probability that the
answer is correct, and a method either rescales a score after the model has produced it or
changes the model that produces it. Rescaling is the standard treatment for
probability-based signals: a temperature or a bias layer is fitted on a labeled held-out
set~\citep{guo2017calibration,liu2024litcab}, and label-free variants align the score with
a reference model or project the two predictive distributions onto a calibrated
one~\citep{luo2025your,kong2026calibration}. Probes need no separate step, since they are
classifiers already trained on correctness
labels~\citep{azaria2023internal,srey2026signals}, and a sampling-based estimate is
calibrated by choosing how agreement among samples maps to a
probability~\citep{kuhn2023semantic,del2026uncertainty}. A verbalized score can be
rescaled the same way, but only through extra queries per question or an extra pass that
rewrites the reported value~\citep{wang2026calibrating,yeh2026retrieval}, and neither
changes how the model generates the score in the first place.

\paragraph{Post-training for verbalized confidence.}
Post-training directly optimizes models to express confidence
alongside their answers. Supervised approaches train on constructed
confidence targets, while reinforcement learning uses rewards
that encourage reliable confidence estimates.
SFT methods derive these targets from correctness signals on the model's own
responses~\citep{lin2022teaching,zhang2024rtuning,jang2025verbalized,liu2024uncertainty,xu2024sayself,hager2025uncertainty,seo2025advice,zhang2026lovec,li2025conftuner}.
Reinforcement learning adds a calibration term to the reward of RLVR, which by itself
optimizes answer correctness and leaves the confidence
unsupervised~\citep{shao2024deepseekmath,yu2025dapo}; the term is a proper scoring rule on
the reported confidence, and the methods differ in where the score is placed, what it is
trained toward, and how credit reaches its
tokens~\citep{leng2025taming,damani2026beyond,bani2026rewarding,tan2026effectiveness,li2026confidence,finlay2026caliber,shi2026sage,yu2026semantic,li2026orce,ma2026decoupling,wang2026process,yang2026scaling,zhang2026collaboration}.
However, on-policy confidence-aware RL relies on confidence values sampled
from the current policy. Our analysis in
Section~\ref{sec:motivation} shows how a concentrated confidence
prior can restrict this exploration.
CalibSFT addresses this limitation by shaping the confidence prior
through supervised fine-tuning before confidence-aware RL.

\section{Theoretical Analysis}

\subsection{Proof of Proposition~\ref{prop:confidence_saturation}}
\label{sec:confidence_saturation_proof}

\begin{proof}
Fix the context representation $h$ and the corresponding expected rewards $u_i(h)$, where $p_h = \Pr(Z = 1 \mid h)$.
Expanding the squared error gives:
\begin{equation}
    u_i(h) = -\mathbb E[(c_i - Z)^2 \mid h] = -\left[ p_h(1 - p_h) + (c_i - p_h)^2 \right].
\end{equation}
Since $c_i, p_h \in [0, 1]$, the reward spread satisfies $\Delta_h = \max_i u_i(h) - \min_i u_i(h) = \max_i (c_i - p_h)^2 - \min_i (c_i - p_h)^2 \le 1$.

Differentiating $J_h(\eta) = \sum_{j=1}^K \pi_j u_j(h)$ with respect to logit $\eta_i$, the softmax derivative yields $\partial \pi_j / \partial \eta_i = \pi_j (\mathbf 1\{j = i\} - \pi_i)$.
It follows that
\begin{align}
    \frac{\partial J_h}{\partial\eta_i}
    &= \sum_{j=1}^K u_j(h) \pi_j (\mathbf 1\{j = i\} - \pi_i) \\
    &= \pi_i u_i(h) - \pi_i \sum_{j=1}^K \pi_j u_j(h) \\
    &= \pi_i \bigl(u_i(h) - J_h\bigr).
\end{align}
Since $J_h$ is a convex combination of $\{u_j(h)\}_{j=1}^K$, we have $\min_j u_j(h) \le J_h \le \max_j u_j(h)$, which implies $|u_i(h) - J_h| \le \Delta_h$.
This establishes the coordinate-wise bound:
\begin{equation}
    \left|\frac{\partial J_h}{\partial\eta_i}\right| \le \pi_i \Delta_h.
\end{equation}
If the policy concentrates on a dominant action $k$ such that $\pi_k = 1 - \varepsilon$, then
\begin{equation}
    |u_k(h) - J_h| = \left| \sum_{i \ne k} \pi_i \bigl(u_k(h) - u_i(h)\bigr) \right| \le \sum_{i \ne k} \pi_i \Delta_h = \varepsilon \Delta_h.
\end{equation}
Consequently, the $L_1$ norm of the gradient satisfies:
\begin{align}
    \|\nabla_\eta J_h\|_1
    &= \pi_k |u_k(h) - J_h| + \sum_{i \ne k} \pi_i |u_i(h) - J_h| \\
    &\le (1 - \varepsilon)\varepsilon \Delta_h + \varepsilon \Delta_h \\
    &= (2 - \varepsilon)\varepsilon \Delta_h \le 2\varepsilon \Delta_h \le 2\varepsilon.
\end{align}
\end{proof}

\paragraph{Connection to confidence-aware RL.}
Under $C \perp Z \mid h$, the gradient is the expectation of the likelihood-ratio estimator $U \nabla_\eta \log \pi(I \mid h)$, where $I \sim \pi(\cdot \mid h)$ and $U = -(c_I - Z)^2$.
In practical algorithms such as RLCR, the total reward is $Z - (c_I - Z)^2$.
The label term $Z$ has conditional expectation $\mathbb E[Z \mid h] = p_h$, which is independent of policy parameters $\eta$ and contributes zero gradient.
Subtracting any baseline $b(h)$ that does not depend on the sampled action $I$ leaves this expectation unchanged.
GRPO's group-mean baseline includes the response's own reward, which scales the expected gradient by $(G-1)/G$ without changing its direction.
Its standard-deviation normalization further introduces a data-dependent scale $1/(\sigma_R+\epsilon)$; treating this scale as fixed within a group, the suppression bound holds up to a positive factor.

\paragraph{Moving beyond a concentrated prior.}
Under the suppression bound, each update moves the logit of a suboptimal dominant value $c_k$ by at most $\mathcal O(\varepsilon)$.
Since policy entropy tends to decrease during RL for reasoning models~\citep{cui2025entropy}, and confidence token entropy remains low throughout our RL runs (Figure~\ref{fig:training_dynamics_support}), $\varepsilon$ stays small and the policy remains concentrated on a few confidence values within the training budget.

\subsection{Proof of Proposition~\ref{prop:confidence_support}}
\label{sec:confidence_support_proof}

\begin{proof}
Recall that $p_h = \Pr(Z = 1 \mid h)$.

Decomposing the error into $C - Z = (C - p_h) - (Z - p_h)$, conditional independence yields:
\begin{equation}
    \mathbb E[(C - p_h)(Z - p_h) \mid h] = \mathbb E[C - p_h \mid h] \cdot \mathbb E[Z - p_h \mid h] = 0,
\end{equation}
since $\mathbb E[Z \mid h] = p_h$. Expanding the squared error therefore gives:
\begin{align}
    \mathbb E[(C - Z)^2 \mid h]
    &= \mathbb E[(Z - p_h)^2 \mid h] + \mathbb E[(C - p_h)^2 \mid h] \\
    &= p_h(1 - p_h) + \mathbb E[(C - p_h)^2 \mid h].
\end{align}
For any $\delta > 0$, the second term satisfies:
\begin{align}
    \mathbb E[(C - p_h)^2 \mid h]
    &\ge \mathbb E\bigl[(C - p_h)^2 \cdot \mathbf 1\{|C - p_h| \ge \delta\} \mid h\bigr] \\
    &\ge \delta^2 \Pr(|C - p_h| \ge \delta \mid h) \\
    &= \delta^2 \bigl(1 - m_\delta(h)\bigr),
\end{align}
where $m_\delta(h) = \Pr(|C - p_h| < \delta \mid h)$. Substituting this back into the equality yields:
\begin{equation}
    \mathbb E[(C - Z)^2 \mid h] \ge p_h(1 - p_h) + \delta^2 \bigl(1 - m_\delta(h)\bigr).
\end{equation}
Thus, assigning little probability mass near $p_h$ incurs an excess Brier loss even when confidence values in that neighborhood remain possible.
\end{proof}

\paragraph{Restricted confidence support.}
As a special case, suppose $\Pr(C \in S_h \mid h) = 1$ for a nonempty set $S_h \subseteq [0, 1]$, and define $d_h = \operatorname{dist}(p_h, S_h) = \inf_{c \in S_h} |c - p_h|$.
If $d_h > 0$, then $m_{d_h}(h) = 0$, and Proposition~\ref{prop:confidence_support} with $\delta = d_h$ gives:
\begin{equation}
    \mathbb E[(C - Z)^2 \mid h] \ge p_h(1 - p_h) + \operatorname{dist}(p_h, S_h)^2.
\end{equation}
For $d_h = 0$, the same bound follows directly from the decomposition above.
If $S_h$ is closed, a predictor that always outputs a nearest point in $S_h$ attains the bound.
When verbalized confidence values concentrate on a small discrete set of high values (Figure~\ref{fig:initial_confidence_prior}), $\operatorname{dist}(p_h, S_h)$ measures the distance from the nearest available value to the true posterior $p_h$, directly translating the concentrated prior into excess Brier risk.

\paragraph{Best-of-$G$ sampling under a shared context.}
GRPO samples one confidence value per response context.
As a more favorable case for exploration, draw $G$ independent confidence samples $C_1, \ldots, C_G$ from $\pi(\cdot \mid h)$ for the same context $h$.
Each sample lies outside $\{c : |c - p_h| < \delta\}$ with probability $1 - m_\delta(h)$.
Since the samples are independent given $h$, all $G$ samples do so with probability $\bigl(1 - m_\delta(h)\bigr)^G$.
On that event, even the sample closest to $p_h$ has squared error at least $\delta^2$, yielding:
\begin{equation}
    \mathbb E\!\left[\min_{i \le G} (C_i - p_h)^2 \,\middle|\, h\right] \ge \delta^2 \bigl(1 - m_\delta(h)\bigr)^G.
\end{equation}
Since a fixed value $c$ has expected Brier loss $p_h(1-p_h) + (c - p_h)^2$, this bounds the excess Brier loss of the best of $G$ samples.
With $G = 8$, the group size in our RL runs, the $G$ samples contain a value within $\delta$ of $p_h$ with probability only $1 - \bigl(1 - m_\delta(h)\bigr)^8$, which remains small when $m_\delta(h) \approx 0$.
Thus, even repeated sampling from the same context rarely reaches values near $p_h$.

\subsection{Proof of Proposition~\ref{prop:confidence_target_moments}}
\label{sec:confidence_target_proof}

We analyze the unrounded targets under the rollout distribution before filtering and balancing.

\begin{proof}
Fix $x$ and denote $v_x = p_x(1 - p_x)$. Under independent rollout sampling, $\mathbb E[Q_x \mid x] = \mathbb E[Z_i \mid x] = p_x$ and $\operatorname{Var}(Q_x \mid x) = v_x / n$.
Since $Z_i$ is included in the summation of $Q_x$, their covariance is $\operatorname{Cov}(Q_x, Z_i \mid x) = v_x / n$.
By linearity of expectation:
\begin{equation}
    \mathbb E[C_i^* \mid x] = \lambda \mathbb E[Q_x \mid x] + (1 - \lambda)\mathbb E[Z_i \mid x] = p_x.
\end{equation}
Expanding the conditional variance yields:
\begin{align}
    \operatorname{Var}(C_i^* \mid x)
    &= \lambda^2 \frac{v_x}{n} + (1 - \lambda)^2 v_x + 2\lambda(1 - \lambda)\frac{v_x}{n} \\
    &= v_x \left[ \frac{1}{n} + (1 - \lambda)^2 \left(1 - \frac{1}{n}\right) \right].
\end{align}
Now fix any prediction $\hat c$ at which $\ell$ is finite. Since $\ell(\hat c, t)$ is affine in $t$, taking expectations gives:
\begin{equation}
    \mathbb E[\ell(\hat c, C_i^*) \mid x] = p_x \ell(\hat c, 1) + (1 - p_x) \ell(\hat c, 0) = \mathbb E[\ell(\hat c, Z_i) \mid x].
\end{equation}
Because $\ell$ is strictly proper, $\mathbb E[\ell(\hat c, Z_i) \mid x]$ is uniquely minimized at $\hat c = p_x$.
Therefore, the mixed target preserves $p_x$ as the unique population minimizer under any strictly proper scoring loss.
\end{proof}

\paragraph{Variance reduction and discrimination margin.}
At $\lambda = 0$, the target reduces to the binary label with maximum variance $v_x$.
At $\lambda = 1$, the target is the question success rate with minimum variance $v_x / n$, but provides zero within-question discrimination.
For any $\lambda \in (0, 1)$, the variance is strictly smaller than $v_x$.
Within any observed rollout group, because $q_x$ is shared across all responses, the target difference between a correct response $i$ ($z_i = 1$) and an incorrect response $j$ ($z_j = 0$) is exactly:
\begin{equation}
    c_i^* - c_j^* = (1 - \lambda)(z_i - z_j) = 1 - \lambda.
\end{equation}
Furthermore, taking expectation over the rollout sampling, the expected target separation between correct and incorrect responses satisfies $\mathbb E[C_i^* \mid Z_i = 1, x] - \mathbb E[C_i^* \mid Z_i = 0, x] = 1 - \lambda + \frac{\lambda}{n} \equiv \alpha$.
With $n = 50$ and $\lambda = 0.5$, the within-group margin is $0.50$ and the expected margin is $\alpha \approx 0.51$, so incorrect responses receive clearly lower targets, while the target variance drops from $v_x$ to about $0.265v_x$, a reduction of over $73\%$ relative to binary labels.

\paragraph{Conditioning on response context and shrinkage.}
Proposition~\ref{prop:confidence_target_moments} establishes question-level calibration under the rollout distribution.
When conditioned on a specific response context $h_i = \phi(x, r_i, a_i)$, let $p_{h_i} = \Pr(Z_i = 1 \mid h_i)$.
Assuming remaining rollouts are conditionally independent of $h_i$ given $x$, the expected target satisfies:
\begin{equation}
    \mathbb E[C_i^* \mid h_i] = \alpha \, p_{h_i} + (1 - \alpha) \, p_x,
    \label{eq:shrinkage_target}
\end{equation}
where $\alpha$ is the expected margin defined above.
Eq.~\eqref{eq:shrinkage_target} shows that the target shrinks the noisy single-sample correctness signal toward the question-level success rate $p_x$.
This shrinkage introduces a response-level bias of $(1-\alpha)(p_x - p_{h_i})$ in exchange for lower target variance, while keeping a margin of $1-\lambda$ between correct and incorrect responses within each group.
These theoretical results characterize the supervision signal under the rollout distribution; the empirical effects of subsequent target balancing and autoregressive fine-tuning are evaluated in Section~\ref{sec:overall_results} and Appendix~\ref{sec:ablation_studies}.

\section{Implementation Details}
\label{sec:detailed_setup}

\subsection{Training Pipelines}
\label{sec:training_pipelines}

We use Qwen3-8B as the default base model throughout the experiments.
For DeepScaleR, we remove duplicate questions
and questions with conflicting reference answers,
resulting in 38{,}915 unique questions.
We randomly partition these questions into
a held-out evaluation set of 2{,}000 questions (DeepScaleR-Eval)
and a training set of 36{,}915 questions (DeepScaleR-Train),
where the latter is used to construct both the SFT dataset for CalibSFT and the downstream RL training data.

\paragraph{SFT data construction.}
\label{sec:calibsft_data_details}
We sample $n=50$ responses per question from DeepScaleR-Train with the base model, using temperature 1.0, top-$p$ 0.95, and a maximum length of 8{,}192 tokens. This yields 1{,}845{,}750 responses in total.
Before filtering, we compute the empirical success rate $q_x$ for each question from all 50 responses.
We then discard responses that fail to follow the \texttt{<think>}/\texttt{<answer>}/\texttt{<confidence>} format or exceed the token limit without finishing.
For each retained response, we compute the confidence target $c_i^*=\lambda q_x+(1-\lambda)z_i$ with $\lambda=0.5$ and replace the original confidence value with $c_i^*$, while keeping reasoning trace and answer unchanged.
In the default probability format, each target is rounded to two decimal places within the \texttt{<confidence>} tag, producing 100 distinct target values.
To perform target-balanced sampling, we group examples by their target value and uniformly sample an equal number of responses from each group, with the number set by the smallest group.
This yields a balanced dataset containing 349 examples per target, totaling 34{,}900 examples across 11{,}539 unique questions for main experiments.

\paragraph{CalibSFT.}
We fine-tune the model on the balanced set using correctness-conditional supervision with $\beta=0.5$,
where correct responses supervise both the reasoning/answer and confidence tokens,
while incorrect responses supervise only the confidence tokens.
The optimizer is AdamW ($\beta_1=0.9$, $\beta_2=0.999$, zero weight decay).
We employ a constant learning rate of $2\times10^{-6}$ after 10 warm-up steps, with a global batch size of 256 sequences truncated to 8{,}192 tokens.
The model is trained for 180 optimization steps on 8 NVIDIA B20Z GPUs.

\paragraph{Confidence-aware RL.}
For downstream reinforcement learning, each prompt group contains 8 rollouts, with each response capped at 8{,}192 tokens.
We train the policy with a batch size of 64 prompt groups for 160 optimization steps.
For methods using DAPO dynamic sampling, which keeps only groups with mixed correctness and discards groups that are entirely correct or incorrect, this corresponds to approximately 600 generation batches.
To accelerate training for these dynamic-sampling methods, we pre-filter the training set by removing questions where all base-model rollouts are already correct.
Each RL baseline follows its designated paradigm while preserving its original reward formulation and credit-assignment mechanism:

\begin{itemize}[leftmargin=*,itemsep=2pt,topsep=2pt]

    \item \textbf{Joint RL baselines (RLCR, CoCA, DCPO, and C3RL).}
        These methods simultaneously optimize answer correctness and confidence rewards starting directly from the base checkpoint ($\text{Base} \rightarrow \text{RL}$), whereas their CalibSFT counterparts initialize from the fine-tuned prior ($\text{Base} \rightarrow \text{CalibSFT} \rightarrow \text{RL}$).
        For RLCR~\citep{damani2026beyond}, CoCA~\citep{li2026confidence}, and DCPO~\citep{ma2026decoupling}, training uses the default probability format along with DAPO dynamic sampling.
        In contrast, C3RL~\citep{yang2026scaling} integrates reference-accuracy rewards using an integer confidence format ($0$--$9$), and its CalibSFT initialization is adapted accordingly.
        To preserve C3RL's category-dependent reference signals (all-correct, partially-correct, or all-incorrect), we train both the baseline and CalibSFT runs on the full dataset without dynamic sampling.

    \item \textbf{Two-stage RL baseline (ReDoubt).}
        Following~\citet{bani2026rewarding}, ReDoubt decouples accuracy and confidence by optimizing a clipped log-score confidence objective on top of an accuracy-aligned policy ($\text{Base} \rightarrow \text{RLVR} \rightarrow \text{ReDoubt}$).
        To seamlessly integrate with this paradigm, we fine-tune the trained RLVR checkpoint with CalibSFT before running ReDoubt RL ($\text{Base} \rightarrow \text{RLVR} \rightarrow \text{CalibSFT} \rightarrow \text{ReDoubt}$).

\end{itemize}

\input{tables/format_prompts.tex}

\subsection{Evaluation Protocol}
For our primary evaluations, we decode at temperature 0.6 and sample 32 responses per question for the AIME benchmarks and 4 responses per question for all other benchmarks.
Pass@1 is reported as the average answer correctness across these sampled outputs.
To evaluate calibration, the predicted confidence value is parsed from the \texttt{<confidence>} tag and normalized to $[0, 1]$, with ECE calculated using 10 equal-width bins.
Unless stated otherwise, all RL evaluations report the final checkpoint.

\paragraph{Answer verification.}
The verification rules yield the binary correctness labels $z_i$ that construct the CalibSFT targets and provide verifiable rewards during downstream RL.
Depending on the benchmark protocol, we use three types of verification:
\begin{itemize}[leftmargin=*,itemsep=2pt,topsep=2pt]
   \item \textbf{Symbolic and normalized matching:} Mathematical benchmarks and HotpotQA verify correctness via symbolic equivalence under \texttt{math\_verify} or exact text match after standard normalization (lowercasing, removing punctuation/articles, and collapsing whitespace). NQOpen and MuSiQue apply the same normalization against any annotated gold alias.
   \item \textbf{Official benchmark protocols:} For datasets with established grading rules, we strictly follow their official scripts. WebQuestions requires an exact set match, DROP applies its official normalization for exact string matching, PopQA accepts predictions containing a gold alias, and LiveBenchReasoning employs its native rule-based graders for reasoning subtasks.
    \item \textbf{LLM-as-a-judge:} Since TriviaQA allows open-ended surface variations, we employ Qwen3-8B as an LLM judge. It takes the question, reference answer, and model prediction as input to render a binary correctness decision.
\end{itemize}

\subsection{Confidence Formats}
\label{sec:format_details}
To evaluate format generality in Section~\ref{sec:generality}, we keep the training data fixed and only alter how the confidence target $c^*$ is expressed.
The corresponding prompts for each format are provided in Table~\ref{tab:format_prompts}.

\begin{itemize}[leftmargin=*,itemsep=2pt,topsep=2pt]
    \item \textbf{Probability:} Expressed as a two-decimal float. The fine grid with step 0.01 is our default format, while the coarse grid with step 0.05 is evaluated as an ablation.

    \item \textbf{Integer:} Mapped to a single digit $\lfloor 9c^* + 0.5 \rfloor \in \{0, \ldots, 9\}$, and normalized by dividing by 9 during evaluation.

    \item \textbf{Linguistic:} Categorized into \emph{low} ($c^* < 1/3$), \emph{medium} ($1/3 \le c^* < 2/3$), and \emph{high} ($c^* \ge 2/3$), which are evaluated at fixed values of 0.25, 0.50, and 0.75, respectively.
\end{itemize}

\section{Detailed Experimental Results}

\begin{figure}[p]
    \centering
    \includegraphics[width=0.85\linewidth]{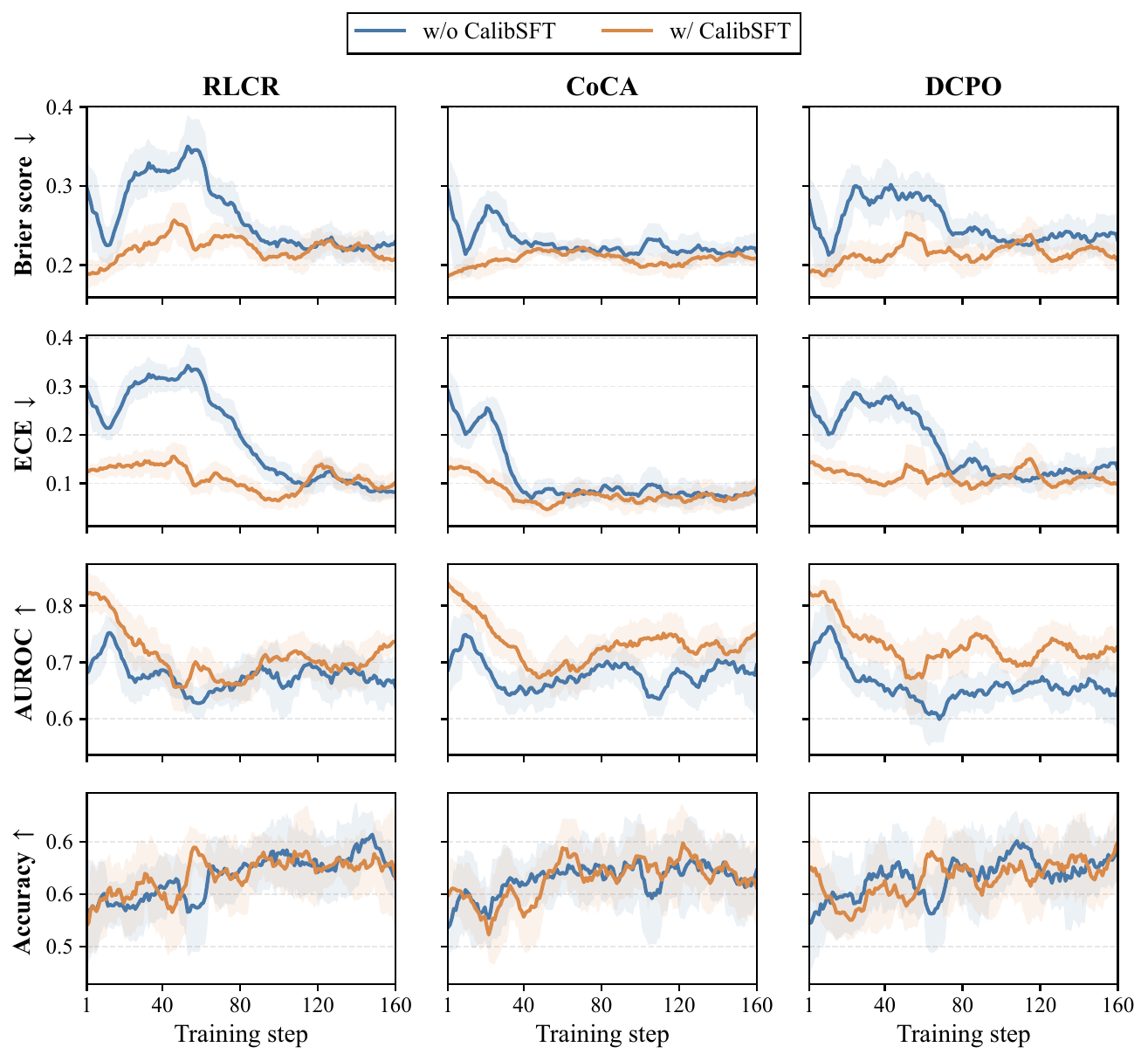}
    \caption{Training dynamics of RLCR, CoCA, and DCPO with and without CalibSFT initialization. Standard RL improves Brier score and ECE but degrades AUROC over training. In contrast, CalibSFT consistently achieves markedly better Brier score, ECE, and AUROC while preserving task accuracy.}
    \label{fig:training_dynamics_calibration}
    \vspace{6pt}
    \includegraphics[width=0.85\linewidth]{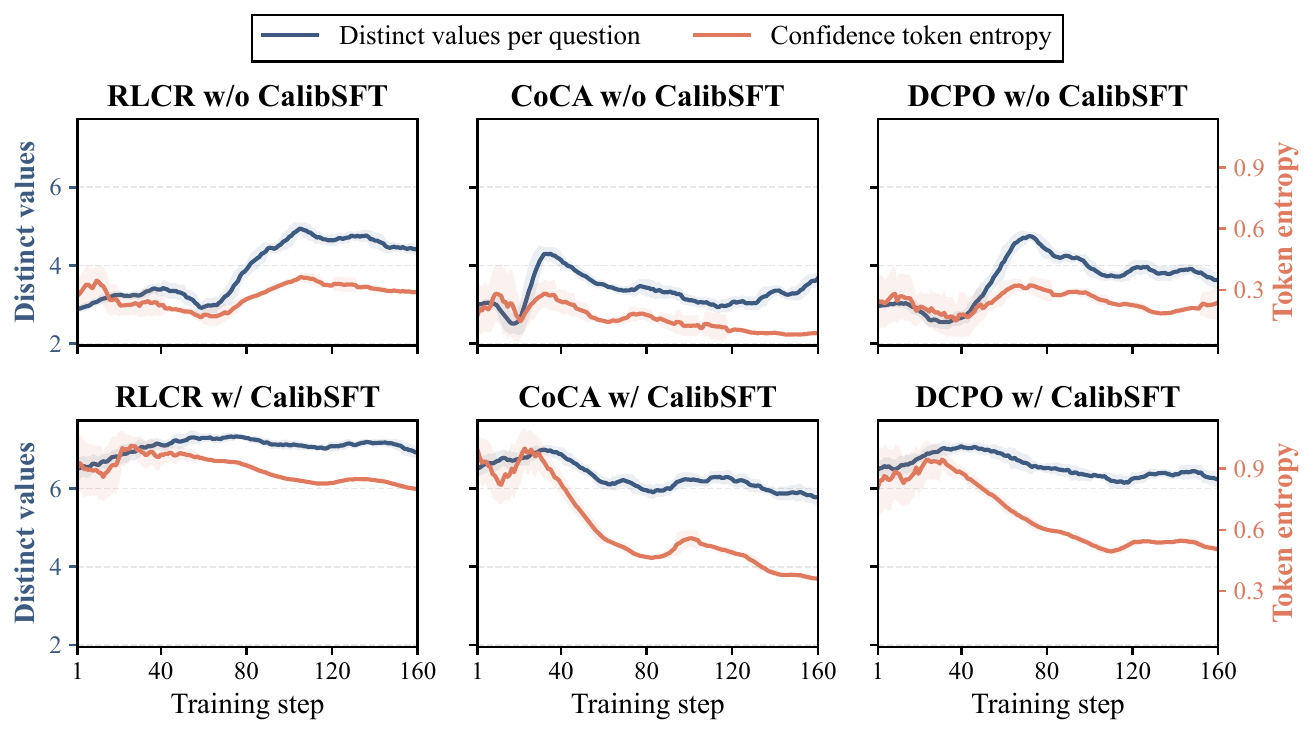}
    \caption{Confidence exploration dynamics during training across methods. The number of distinct confidence values sampled per question (left axis) closely tracks confidence token entropy (right axis).  In standard RL (top), confidence token entropy remains low, restricting the number of distinct confidence values sampled per question. In contrast, CalibSFT initialization (bottom) substantially elevates token entropy throughout policy optimization, doubling distinct sampled values and enabling active exploration across the confidence spectrum.}
    \label{fig:training_dynamics_support}
\end{figure}

\subsection{Training Dynamics}
\label{sec:training_dynamics}

In Section~\ref{sec:motivation}, we observed that confidence remains concentrated on a few high values when training RLCR with its calibration reward.
To examine whether this limitation is universal, we further train Qwen3-8B with CoCA~\citep{li2026confidence} and DCPO~\citep{ma2026decoupling}, tracking their training trajectories with and without CalibSFT initialization.

\paragraph{CalibSFT mitigates the calibration--discrimination trade-off.}
As illustrated in Figure~\ref{fig:training_dynamics_calibration}, this calibration-discrimination tension is indeed a shared failure mode across standard confidence-aware RL methods.
Across all three baselines, optimizing the calibration reward successfully lowers calibration error (e.g., ECE drops from 23.8--25.8\% to 7.8--14.2\%), yet AUROC steadily degrades (falling from 70.8--73.5\% to 64.1--68.0\%).
Consequently, the policy merely minimizes the aggregate calibration penalty without learning to distinguish correct from incorrect answers.
In contrast, initializing policy optimization with CalibSFT sustains lower Brier score and superior AUROC throughout training while preserving task accuracy.

\paragraph{CalibSFT broadens confidence exploration during RL.}
Figure~\ref{fig:training_dynamics_support} tracks the corresponding
confidence exploration.
Without CalibSFT, confidence-token entropy remains low, and each
question produces fewer than 5.3 distinct confidence values on
average across the eight rollouts.
With so few distinct values, incorrect responses rarely receive low
confidence during rollouts, leaving the policy gradient little signal to
lower their confidence.
CalibSFT increases token entropy by more than threefold and roughly
doubles the number of distinct confidence values sampled per
question throughout training.
The resulting broader support is consistent with the improved
discrimination shown in Figure~\ref{fig:training_dynamics_calibration}.

\input{tables/per_dataset_results.tex}

\subsection{Per-Dataset Main Results}
\label{sec:per_dataset_results}
Tables~\ref{tab:appendix_id_results} and~\ref{tab:appendix_ood_results}
present the comprehensive results across all 16 individual benchmarks,
complementing the split-level averages in Table~\ref{tab:main_results}.
The overall average across the 8 datasets within each split is summarized in the final subtable.

\paragraph{CalibSFT improves calibration on difficult benchmarks.}
Across individual tasks, CalibSFT initialization lowers ECE in 61 of 80 method--dataset pairs,
with the most pronounced improvements where baseline RL models suffer from severe overconfidence.
Specifically, on the six benchmarks where baseline RL methods
exhibit an ECE exceeding 25\% (MinervaMath, MuSiQue, LiveBenchReasoning, NQOpen, PopQA, and WebQuestions),
initializing with CalibSFT consistently reduces both Brier score and ECE across all five RL algorithms.
On high-accuracy tasks such as DROP, baseline RL models achieve deceptively low ECE (e.g., 2.30\% for DCPO) because their confidence clusters around the task accuracy, yet their discrimination remains weak (AUROC $\approx 59\%$).
CalibSFT breaks this uninformative concentration and substantially improves discrimination (e.g., DCPO's AUROC rises from 59.34\% to 75.92\%).
While RLVR improves Pass@1, it leaves calibration largely unaddressed and even worsens OOD Brier score and ECE relative to the base model.
In contrast, CalibSFT alone substantially lowers Brier score and ECE across both splits, effectively shifting the concentrated confidence prior toward a well-calibrated scale.
Although CalibSFT alone exhibits limited OOD discrimination, subsequent confidence-aware RL effectively recovers this ranking capability, lifting OOD AUROC from 61.20\% to over 71\% for four of the five methods.
Ultimately, CalibSFT establishes a better-scaled confidence prior, while confidence-aware RL improves discrimination outside the training domain.

\subsection{Robustness across Random Seeds}
\label{sec:seed_stability}

The main results report single-run evaluations per method.
To confirm that the improvements from CalibSFT consistently exceed training variance,
we repeat RLCR and DCPO on Qwen3-8B across three random seeds (43, 44, and 45) under the protocol of Table~\ref{tab:main_results}.

\paragraph{CalibSFT improvements are stable across random seeds.}
As shown in Table~\ref{tab:seed_stability}, training variance across runs is minimal.
Across all evaluated seeds, CalibSFT consistently outperforms RL baselines on all calibration and discrimination metrics
across both in-distribution and out-of-distribution benchmarks.
Meanwhile, task accuracy (Pass@1) remains comparable to the baseline.
These results confirm that the improvements brought by CalibSFT are stable.

\input{tables/seed_stability.tex}

\subsection{Evaluation on Additional Metrics}
\label{sec:additional_metrics}

While Table~\ref{tab:main_results} evaluates core accuracy,
calibration, and discrimination,
we further assess the policies along two complementary dimensions:
fine-grained calibration and confidence diversity.
All metrics are computed on the same results from the main evaluation.

\paragraph{Fine-grained calibration.}
To thoroughly evaluate confidence reliability,
we consider 4 additional metrics that evaluate question-level calibration, failure prediction, and calibration ranking:
\begin{itemize}
    \item \textbf{Positive Calibration Error (PCE)}~\citep{ma2026decoupling} focuses on overconfidence by restricting ECE to bins where average confidence is higher than accuracy.

    \item \textbf{Capability Brier score (Cap.\ Brier)}~\citep{yang2026calibration} measures how well a question's mean confidence matches its empirical rollout success rate, tracking calibration against problem difficulty.

    \item \textbf{Excess Area Under the Risk--Coverage Curve (E-AURC)}~\citep{geifman2019bias} evaluates failure prediction by comparing the model's risk--coverage curve to a perfect ranking at the same accuracy.

    \item \textbf{Rank Calibration Error (RCE)}~\citep{huang2024uncertainty} measures whether confidence increases monotonically with expected correctness across equal-mass bins, which is independent of model accuracy and score scales.
\end{itemize}

\input{tables/additional_metrics.tex}

\paragraph{CalibSFT improves confidence quality beyond standard metrics.}
As shown in Table~\ref{tab:additional_metrics}, CalibSFT initialization greatly outperforms standard RL baselines across these metrics.
In particular, the large drop in RCE indicates that higher confidence more reliably reflects higher accuracy.
Meanwhile, the sharp decrease in out-of-distribution PCE shows that the method effectively lowers overconfident predictions on unfamiliar problems, without merely shifting all confidence scores down.

\paragraph{Confidence diversity.}
To examine whether models explore a broader range of confidence values, we evaluate 4 diversity metrics:
\begin{itemize}
    \item \textbf{Support Size} is the number of unique confidence values generated across all responses in a split, directly measuring the size of the confidence support.

    \item \textbf{Top-1} and \textbf{Top-3} are the percentage of responses falling on the single and three most frequent confidence values, tracking how heavily the distribution concentrates on a few points.

    \item \textbf{Effective Support} is the exponential of the confidence entropy ($\exp(H)$), which measures the effective number of values used by the model. Unlike raw support size, this metric is not inflated by rare values that appear only once or twice.
\end{itemize}

\input{tables/diversity_metrics.tex}

\paragraph{CalibSFT broadens the confidence distribution.}
As shown in Table~\ref{tab:diversity_metrics}, CalibSFT universally expands confidence support across all five RL frameworks on both evaluation splits.
In the baseline RL models, predictions collapse onto a tiny set of values.
For instance, the top three values often account for $70\%$ to $90\%$ of all responses, yielding an effective support of fewer than 7 values.
CalibSFT breaks this concentration: Top-1 and Top-3 shares drop sharply, while both Support Size and Effective Support increase by several times (for example, Effective Support rises from $7.03$ to $72.83$ for RLCR out-of-distribution).
Because both Support Size and Effective Support increase together, the model uses many different confidence values regularly, rather than generating a few rare numbers by chance.
Finally, this broader support is consistent with the improved calibration observed in Tables~\ref{tab:additional_metrics} and~\ref{tab:confidence_correlation}.

\subsection{Ablation Studies}
\label{sec:ablation_studies}

\subsubsection{Confidence Target Construction}
\label{sec:target_construction_setup}

\input{tables/training_target_distribution.tex}

The construction of confidence targets in CalibSFT is inspired by proper scoring rules, which serve as a foundation for calibration.
To verify the effectiveness of this target design, we compare CalibSFT
with targets based on group-level correctness (CSFT)~\citep{jang2025verbalized},
token probabilities (UaIT)~\citep{liu2024uncertainty}, semantic consistency
(SaySelf)~\citep{xu2024sayself}, and rubric-based LLM judgments
(LoVeC)~\citep{zhang2026lovec}, using a shared SFT training set and the same
optimization objective.

For a fair comparison, we construct the training data from a shared response
pool sampled uniformly across empirical success rates so that the comparison
covers both easier and harder questions. For each question in
DeepScaleR-Train, we use 10 responses generated by Qwen3-8B, with $z_i\in\{0,1\}$
denoting correctness and $q_x=\frac{1}{10}\sum_{i=1}^{10}z_i$ the empirical
success rate. We select 300 questions at each $q_x\in\{0.1,0.2,\ldots,0.9\}$
and retain one correct and one incorrect response per question, yielding
5{,}400 responses from 2{,}700 questions with balanced correctness at each
success rate. SaySelf uses only its correct responses following its SFT setup.
On this shared subset, we construct the five confidence targets as follows:

\begin{itemize}
    \item \textbf{CSFT.} Both responses receive the group success rate, $c_i^*=q_x$.

    \item \textbf{UaIT.} $c_i^*=\exp(-\mathrm{NLL}_i)$, where $\mathrm{NLL}_i$ is the mean negative log-likelihood of the sampled response tokens under the generating model.

    \item \textbf{SaySelf.} We embed the ten responses to each question with Qwen3-Embedding-8B and cluster them at a cosine similarity threshold of 0.98. Each retained correct response is assigned the frequency of its cluster, $c_i^*=|\mathcal M_i|/10$, where $\mathcal M_i$ is the cluster containing response $i$.

    \item \textbf{LoVeC.} An LLM judge scores each response on a 0--10 rubric (10: correct with sound reasoning; 8--9: correct with minor gaps; 5--7: partially correct or ambiguous; 2--4: meaningful but incorrect; 0--1: invalid or clearly wrong), normalized as $c_i^*=s_i/10$.

    \item \textbf{CalibSFT.} Following our main design, the target combines group success rate with response correctness, $c_i^*=0.5q_x+0.5z_i$.
\end{itemize}

All variants are initialized from Qwen3-8B and use the same supervision objective:
reasoning and answer tokens are supervised on correct responses, while
confidence tokens are supervised on all responses. We train for 80 steps
(40 for SaySelf) with a batch size of 256 and a learning rate of
$2\times10^{-6}$. Table~\ref{tab:training_target_distribution} reports the
proportion of training targets in each of ten equal-width confidence
intervals across the five methods.

\input{tables/target_construction_lambda.tex}

\noindent\textbf{Proper-scoring targets preserve both calibration and discrimination.}
Table~\ref{tab:target_construction} reports the AUROC and Brier score of the constructed targets before SFT (Pre-SFT targets), alongside the resulting models' evaluation performance on in-distribution mathematical benchmarks after SFT (Post-SFT predictions).
In the pre-SFT phase, the metrics evaluate the theoretical properties and supervision quality of the training targets themselves.
Prior methods struggle to balance calibration and discrimination during target construction.
CSFT estimates question-level accuracy $q_x$ from verification signals but assigns the same target to all responses of a question; since each question contributes one correct and one incorrect response, its targets yield an AUROC of 50.00 and a high Brier score of 31.67.
Heuristic targets derived from token probabilities (UaIT), semantic consistency (SaySelf), or LLM judgments (LoVeC) lack grounding in strictly proper scoring rules, which leads to miscalibrated targets and high Brier scores.
In contrast, by combining group success rates with individual response correctness $z_i$, CalibSFT separates correct from incorrect responses by construction at $\lambda=0.5$ (an AUROC of 100.00), and its Brier score equals $\lambda^2$ times that of question-level targets (7.92 vs.\ 31.67).

\noindent\textbf{Target-level advantages translate to superior post-SFT performance.}
These advantages at target construction transfer to the post-SFT models.
Averaged over eight mathematical benchmarks, the policy fine-tuned with CalibSFT achieves the strongest performance on both metrics.
Specifically, CalibSFT obtains the highest AUROC of 79.68 and the lowest Brier score of 17.74, outperforming models trained on alternative targets.
These results confirm that grounding target construction in proper scoring rules yields supervision that improves both calibration and discrimination of the trained model.

\subsubsection{Target-Balanced Sampling}

Our CalibSFT pipeline balances training responses across confidence targets. This step prevents frequent targets (especially $0$ and $1$) from dominating the loss and pushing predicted confidences toward the endpoints. To verify this effect, we compare our balanced model against a baseline trained on the raw, unbalanced responses from the same $n=50$ candidate pool, evaluating both on DeepScaleR-Eval.

\paragraph{Target balancing mitigates concentration at extremes.}
As shown in Figure~\ref{fig:balance_ablation_confidence} and Table~\ref{tab:balance_ablation}, balancing has little impact on standard calibration error, but substantially improves confidence diversity. Specifically, it drops the mass at the extremes from $75.56\%$ to $14.49\%$ and expands the effective support from $3.33$ to $29.86$. This broader distribution is critical for downstream RL: instead of collapsing into near-binary actions ($0$ or $1$), the policy gains access to a continuous, fine-grained action space for optimization, consistent with our diversity analysis in Table~\ref{tab:diversity_metrics} and format comparisons in Table~\ref{tab:format_calibsft}.

\input{tables/balance_ablation.tex}

\subsubsection{Target Mixing Weight}
We construct confidence targets as
$c_i^*=\lambda q_x+(1-\lambda)z_i$, where $q_x$ is the group
success rate and $z_i$ is the correctness of an individual response.
To evaluate the effect of the mixing weight, we compare
$\lambda\in\{0,0.25,0.5,0.75,1\}$ using the same subset as above
and keeping the SFT settings fixed.
At $\lambda=0$, targets are restricted to 0 and 1; this setting yields
the highest Brier score in Table~\ref{tab:lambda_ablation}.
At $\lambda=1$, correct and incorrect responses to the same question
receive identical targets, providing no within-question discrimination
signal; post-SFT AUROC is also lower than at $\lambda=0.5$.
For $\lambda>0.5$, a correct response to a hard question can receive a lower target than an incorrect response to an easy question, since $\lambda q_h+(1-\lambda)<\lambda q_e$ whenever $q_e-q_h>(1-\lambda)/\lambda$; post-SFT AUROC accordingly drops from 79.68 at $\lambda=0.5$ to 75.93 at $\lambda=0.75$.
In contrast, any $\lambda\le 0.5$ keeps every correct target above every incorrect one, because $\lambda(q_e-q_h)<\lambda\le 1-\lambda$.
Equal weighting is thus the largest weight on the success rate that preserves this ordering, and it achieves the highest AUROC and lowest Brier score on ID benchmarks among the tested settings, so we use $\lambda=0.5$ as the default.

\input{tables/supervision_ablation.tex}

\subsubsection{Correctness-Conditional Supervision}
To verify the effectiveness of the CalibSFT design, we ablate three choices in
the SFT loss: restricting reasoning and answer supervision to correct
responses, supervising confidence on all responses, and balancing the
segment losses. Table~\ref{tab:supervision_ablation} reports the results.

\noindent\textbf{Correctness-conditional supervision improves calibration and discrimination.}
Confidence-only supervision leaves Pass@1 far behind the other variants,
showing that reasoning and answer supervision is necessary for accuracy.
Without segment balancing, calibration is markedly worse: Full-CE trails
Full-Balanced on both AUROC and ECE. Restricting
reasoning and answer supervision to correct responses slightly raises
accuracy, consistent with incorrect reasoning pulling the model toward
incorrect solutions, but restricting confidence supervision to correct
responses as well narrows its coverage and weakens calibration.
Supervising confidence on all responses while keeping reasoning and
answer supervision correct-only recovers this coverage: CalibSFT
achieves the best AUROC, Brier score, and ECE among all variants, with
Pass@1 close to the best. This combination best balances accuracy,
calibration, and confidence diversity, giving a well-calibrated starting point for downstream
confidence-aware RL.

\subsection{Integration with Other Confidence Estimators}
\label{sec:confidence_estimators}
While this work primarily studies verbalized confidence,
confidence can also be derived from internal token probabilities or sample consistency.
To evaluate whether the benefits of CalibSFT generalize across different confidence estimation paradigms,
we examine two widely adopted alternatives using the exact response samples and evaluation protocol from the main experiment (Table~\ref{tab:main_results}):

\begin{itemize}
    \item \textbf{Token probability.} We extract model confidence directly from the output logits over confidence tokens.
    Specifically, following the prefix \texttt{<confidence>0.}, we apply a restricted softmax over the ten digit tokens $\{0,\ldots,9\}$ to obtain normalized probabilities $P(d)$, computing the expected confidence as:
        \begin{equation*}
            c = \sum_{d=0}^{9} P(d)\left(\frac{d}{10} + 0.05\right),
        \end{equation*}
        where each bin is represented by its midpoint.

    \item \textbf{Self-consistency.} Following majority voting~\citep{wang2023selfconsistency},
    we aggregate multiple rollouts per question (4 by default; 32 on AIME) by answer equivalence.
    The dominant consensus answer determines final correctness, and its corresponding confidence is defined as the mean verbalized score across all rollouts in this consensus group.

\end{itemize}

\paragraph{CalibSFT transfers across confidence estimators.}
Table~\ref{tab:confidence_estimators} compares these estimators alongside verbalized confidence.
Crucially, CalibSFT initialization consistently improves all metrics across every estimator on both in-distribution and out-of-distribution benchmarks.
For instance, under token probability, CalibSFT boosts in-distribution AUROC from 83.60 to 87.66 while cutting out-of-distribution Brier score from 31.22 to 22.80.
Under self-consistency, it reduces ECE substantially on both splits (from 17.62 to 9.67 ID and from 30.53 to 18.19 OOD).

\input{tables/confidence_estimators.tex}

\subsection{Confidence Visualization}
\label{sec:confidence_visualization}

To intuitively illustrate how CalibSFT alleviates calibration errors across tasks of varying difficulty,
we examine the reliability diagrams of RLCR with and without CalibSFT initialization.
Specifically, we evaluate on three in-distribution mathematical benchmarks of distinct difficulty
(MATH-500, AIME 2026, and MinervaMath; Figure~\ref{fig:reliability_id})
and three out-of-distribution reasoning benchmarks (LiveBenchReasoning, TriviaQA, and NQOpen; Figure~\ref{fig:reliability_ood})
across ten equal-width confidence intervals.

\paragraph{CalibSFT improves calibration across task difficulties.}
As shown in Figures~\ref{fig:reliability_id} and~\ref{fig:reliability_ood},
CalibSFT effectively mitigates both overconfidence and underconfidence across tasks with different difficulty levels.
On harder benchmarks where standard RLCR is heavily overconfident, its average confidence exceeds empirical accuracy by 23.0 points on
MinervaMath, 52.4 points on NQOpen, and 29.1 points on LiveBenchReasoning.
Initializing with CalibSFT reduces the gap between confidence and empirical accuracy,
lowering ECE from 23.05\% to 12.00\% on MinervaMath, from 52.43\% to 30.89\% on NQOpen, and from 29.09\% to 7.75\% on LiveBenchReasoning.
Meanwhile, on easier datasets such as MATH-500 where RLCR exhibits underconfidence,
CalibSFT properly raises the confidence scores and reduces ECE from 7.53\% to 3.89\%.

\begin{figure}[!tp]
    \centering
    \includegraphics[width=\linewidth]{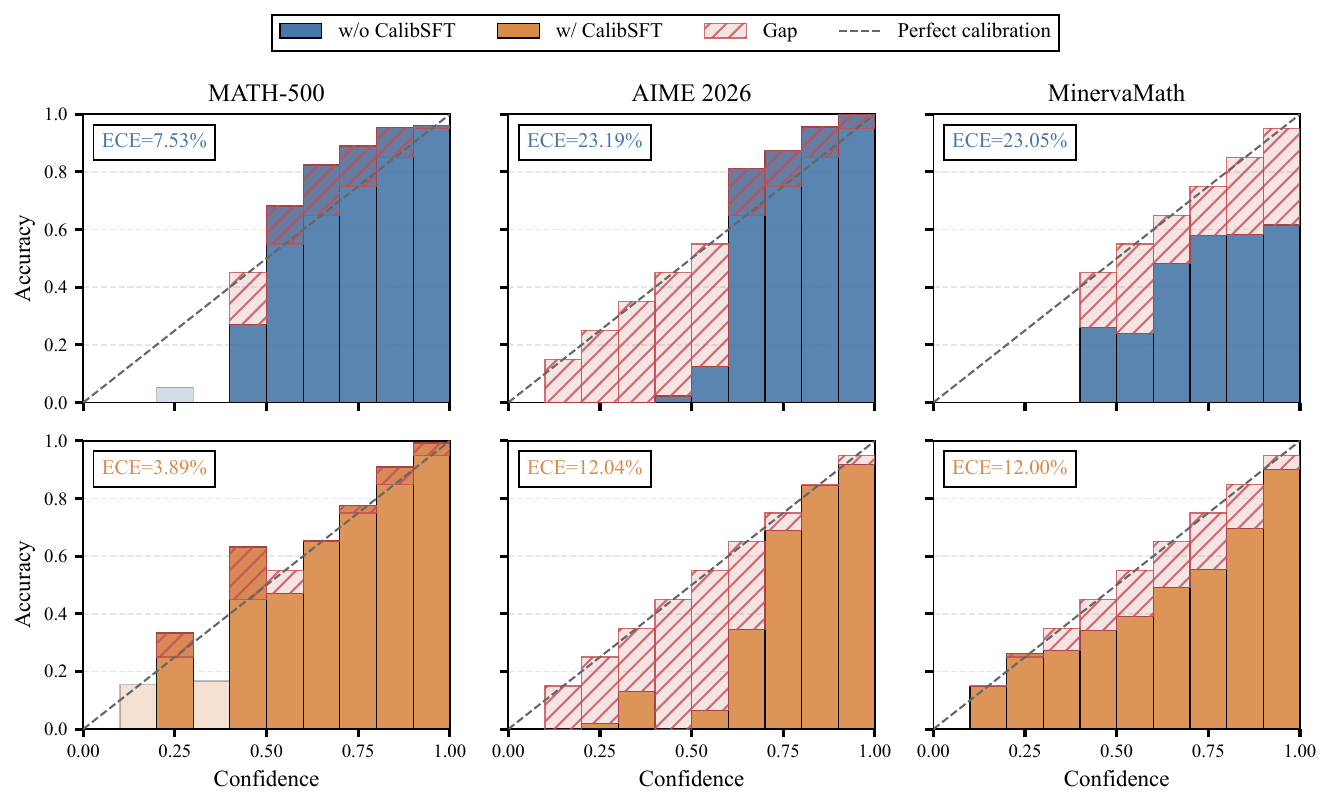}
    \caption{Reliability diagrams of RLCR with and without CalibSFT across three in-distribution mathematical benchmarks.
    CalibSFT effectively reduces both underconfidence on easier tasks (MATH-500)
    and overconfidence on harder tasks (MinervaMath and AIME 2026), aligning confidence closer to actual accuracy.}
    \label{fig:reliability_id}
    \vspace{8pt}
    \includegraphics[width=\linewidth]{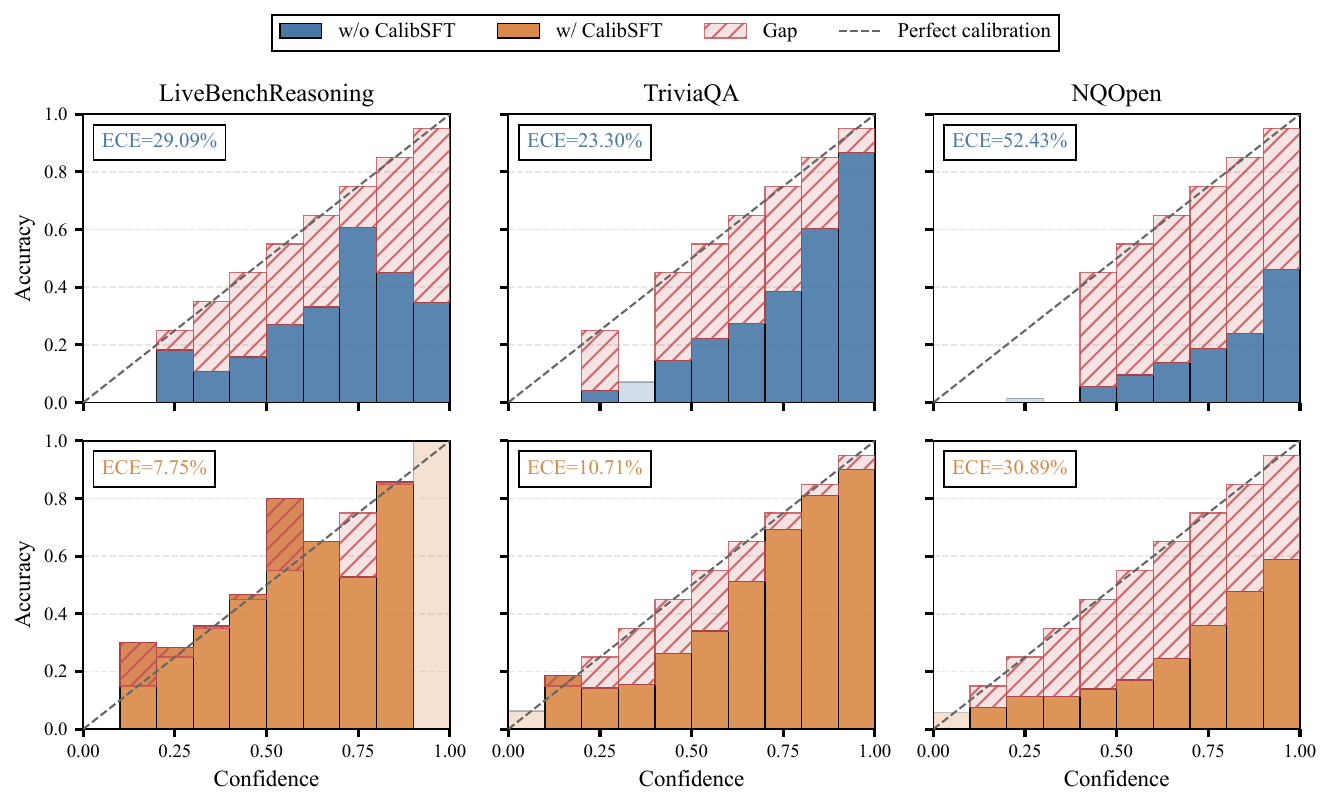}
    \caption{Reliability diagrams of RLCR with and without CalibSFT on out-of-distribution reasoning benchmarks.
    CalibSFT substantially alleviates severe overconfidence across all three tasks and
    prevents confidence from collapsing onto a few extreme values.}
    \label{fig:reliability_ood}
\end{figure}

\subsection{Running Time Analysis}
\label{sec:training_cost}
To evaluate the computational efficiency of CalibSFT,
we compare its training overhead against standard confidence-aware RL.
CalibSFT requires no manual annotation, as its targets are derived purely from self-sampled rollouts
verified by the same automated verifier used in downstream RL.
Crucially, these rollout responses are collected only once before fine-tuning,
whereas RL needs to sample new responses at every optimization step.

CalibSFT relies on a one-time offline rollout stage prior to policy fine-tuning.
For $n=50$, we generate 50 responses per training question, requiring approximately 147 GPU-hours in total (roughly 16.4 wall-clock hours on an 8$\times$ NVIDIA B20Z node).
For a given base model and prompt format, this curated dataset can be reused across all downstream confidence-aware RL methods, avoiding repeated rollout collection.
Reducing the rollout budget to $n=10$ lowers the collection cost to approximately 32 GPU-hours (around 4 wall-clock hours), providing a compute-efficient alternative.
This setting yields probability targets on the grid
$\{0.00,0.05,\ldots,1.00\}$ and remains effective: the corresponding
CalibSFT model achieves an AUROC of 79.68 and a Brier score of 17.74 in
the target-construction experiment (Table~\ref{tab:target_construction}).

\paragraph{CalibSFT requires little fine-tuning time.}
As shown in Table~\ref{tab:training_cost},
CalibSFT introduces minimal overhead compared to downstream RL.
On a single node with 8 NVIDIA B20Z GPUs, each CalibSFT step takes only 0.33 minutes, compared to 4.77 minutes per step for RLCR.
Consequently, the CalibSFT fine-tuning stage accounts for merely 6.9\% of the GPU-hours required by RLCR. The dominant cost is the one-time rollout collection, which is shared across all downstream confidence-aware RL methods.

\input{tables/training_cost.tex}

\section{Downstream Applications}

Beyond standard calibration and discrimination metrics,
we evaluate whether CalibSFT's gains
transfer into practical benefits for downstream decision-making under finite-sample statistical guarantees.
We investigate two deployment scenarios:
(1) \textbf{selective risk control} (\S\ref{sec:selective_risk_control}),
which maximizes response coverage while strictly controlling the selective risk (error rate)
among accepted answers, and
(2) \textbf{confidence-guided model routing} (\S\ref{sec:model_routing}),
which minimizes calls to a stronger and more expensive model
while provably bounding accuracy loss relative to a full-deferral baseline.

\input{tables/risk_control.tex}

\begin{figure}[t]
    \centering
    \includegraphics[width=\linewidth]{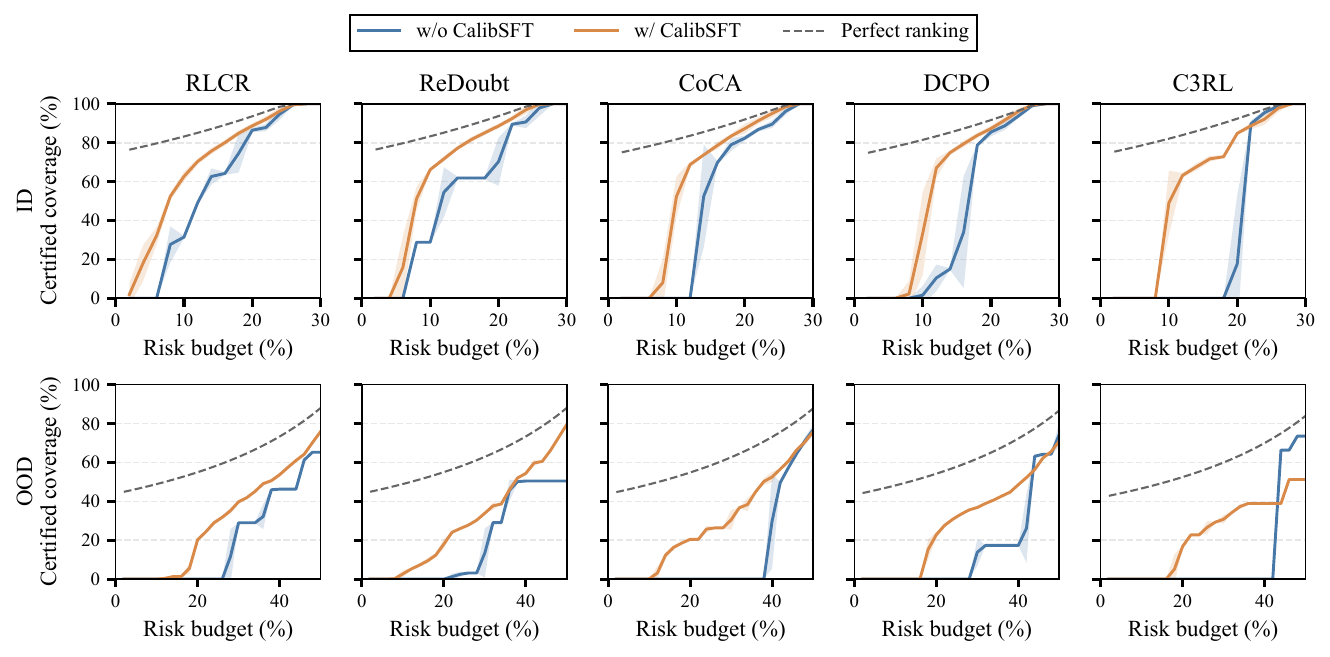}
    \caption{Certified coverage across risk budgets on ID (top) and OOD (bottom) benchmarks.
    The results are averaged over ten random splits.
    Initializing with CalibSFT consistently yields higher certified coverage across budgets and tracks the perfect-ranking
    reference curve more closely, particularly under strict risk tolerances.}
    \label{fig:risk_control}
\end{figure}

\subsection{Selective Risk Control}
\label{sec:selective_risk_control}

Selective prediction allows a reasoning model to answer a question only when its confidence is high enough, and abstain otherwise to avoid incorrect answers.
Specifically, given a threshold $\tau$, the model outputs an answer only if its confidence satisfies $c \ge \tau$.
Under this setting, \emph{coverage} is the fraction of answered questions, while \emph{selective risk} is the error rate among them.
The goal is to maximize coverage while keeping selective risk below a given budget $r_{\max}$~\citep{huang2025model}.
We test whether CalibSFT yields higher certified coverage under the same risk budget.

To select a valid threshold with finite-sample guarantees, we adopt the Learn-then-Test framework~\citep{angelopoulos2024learn} following \citet{shen2026provable}.
Specifically, over a grid of 101 candidate thresholds in $\{0, 0.01, \ldots, 1\}$, we compute a one-sided Clopper--Pearson upper bound on selective risk~\citep{clopper1934use} at confidence level $1-\delta/101$, with $\delta=0.05$.
By a union bound, these risk bounds hold simultaneously across all candidates with probability at least $1-\delta$.
We then choose the smallest threshold whose upper bound is at most $r_{\max}$, maximizing coverage; if no candidate qualifies, the model abstains on all questions.
For both ID and OOD benchmarks, we randomly split the questions into equal calibration and test sets ($50\%/50\%$).
We use one response per calibration question to determine the threshold, and evaluate coverage and selective risk on all responses in the test set.
We repeat this procedure over ten random splits and report the mean and standard deviation of coverage.

\paragraph{CalibSFT increases coverage under a fixed error-rate guarantee.}
Table~\ref{tab:risk_control} shows that CalibSFT initialization substantially boosts certified coverage across all evaluated RL algorithms under identical risk budgets.
On ID benchmarks at a $10\%$ risk budget, CalibSFT raises RLCR's coverage from $31.4\%$ to $62.6\%$, and enables CoCA and C3RL to reach $52.3\%$ and $49.2\%$ coverage, respectively, compared with zero coverage for both baselines.
On OOD benchmarks at a $20\%$ risk budget, all five baselines yield zero certified coverage, whereas CalibSFT achieves coverage ranging from $16.8\%$ to $22.8\%$.
Figure~\ref{fig:risk_control} further illustrates that CalibSFT consistently brings certified coverage closer to the perfect-ranking reference across risk budgets, especially under stricter risk constraints.
These results indicate that mitigating confidence concentration is important for risk-controlled deployment, allowing the model to certify more responses rather than conservatively abstaining.

\input{tables/cascade.tex}

\subsection{Confidence-Guided Model Routing}
\label{sec:model_routing}

Confidence-guided model routing uses a local model's confidence to
decide whether to retain its answer or defer the question to a stronger,
more expensive model. The local model first generates an answer and a
confidence score for every question. Its answer is retained when the
confidence is at least a threshold $\tau$; otherwise, the question is
sent to Qwen3.7-Flash. Following work on efficient reasoning and risk-aware
routing~\citep{yu2026anytime,zeng2026on,hao2026racer}, we aim to
reduce reliance on the stronger model while keeping the accuracy loss
relative to deferring all questions below a tolerance $\varepsilon$.
We measure cost using the fraction of questions deferred and the
percentage of Qwen3.7-Flash tokens saved relative to this baseline.

To certify the routing threshold $\tau$,
we bound the probability that a question is answered locally and incorrectly while the stronger model would answer it correctly,
which limits the accuracy loss relative to deferring all questions to the stronger model.
Following \S\ref{sec:selective_risk_control}, we compute one-sided Clopper--Pearson upper bounds~\citep{clopper1934use} with a union-bound correction across 102 candidate thresholds at significance level $\delta=0.05$, selecting the smallest threshold whose upper bound does not exceed $\varepsilon$.
In our experiments, the local Qwen3-8B is instantiated by each RL policy, while the stronger model is Qwen3.7-Flash.
We evaluate on 4{,}828 questions from eight ID mathematics benchmarks using one rollout per question from each model, randomly splitting the questions into equal calibration and test sets across ten independent trials to report the mean results along with the standard deviation of the deferral rate.

\paragraph{CalibSFT improves routing efficiency under an accuracy-loss guarantee.}
Table~\ref{tab:cascade} shows that CalibSFT initialization consistently lowers deferral rates across all evaluated RL methods under both accuracy-loss tolerances.
Across these baselines, CalibSFT achieves an average deferral reduction of $12.3$ and $27.7$ percentage points at $\varepsilon=1\%$ and $\varepsilon=2\%$, respectively.
In particular, at $\varepsilon=2\%$, CalibSFT cuts DCPO's deferral rate from $82.9\%$ to $33.3\%$, thereby boosting token savings from $5.6\%$ to $43.5\%$.
Across both tolerances, routing guided by CalibSFT maintains overall accuracies between $82.42\%$ and $82.81\%$, which closely match the full-deferral baseline ($82.58\%$).
By lowering confidence on incorrect responses, CalibSFT enables the local model to reliably retain valid solutions, substantially reducing computational overhead in cascaded reasoning without sacrificing accuracy.

\end{document}

%% file: math_commands.tex
\usepackage{amsmath,amsfonts,bm}

\def\eqref#1{equation~\ref{#1}}

\def\1{\bm{1}}

\DeclareMathAlphabet{\mathsfit}{\encodingdefault}{\sfdefault}{m}{sl}
\SetMathAlphabet{\mathsfit}{bold}{\encodingdefault}{\sfdefault}{bx}{n}



%% file: tables/main_results.tex
\begin{table*}[t]
    \centering
    \caption{Performance comparison on in-distribution and out-of-distribution benchmarks. Each entry
    averages the datasets within the corresponding split. CalibSFT
    initialization improves calibration and discrimination for all RL baselines while also
    raising Pass@1.}
    \label{tab:main_results}
    \resizebox{0.98\textwidth}{!}{%
    \renewcommand\arraystretch{1.0}
    \begin{tabular}{lcccccccc}
        \toprule
        & \multicolumn{4}{c}{In-distribution} & \multicolumn{4}{c}{Out-of-distribution} \\
        \cmidrule(lr){2-5} \cmidrule(lr){6-9}
        Method & Pass@1 $\uparrow$ & AUROC $\uparrow$ & Brier $\downarrow$
        & ECE $\downarrow$
        & Pass@1 $\uparrow$ & AUROC $\uparrow$ & Brier $\downarrow$
        & ECE $\downarrow$ \\
        \midrule
        Base & 59.32 & 71.90 & 33.20 & 34.36 & 45.36 & \textbf{65.71} & 45.15 & 46.94 \\
        RLVR & \textbf{64.16} & 73.16 & 31.10 & 31.52 & 43.88 & 64.19 & 50.29 & 51.52 \\
        \rowcolor{calibrow} CalibSFT & 58.51 & \textbf{79.93} & \textbf{17.30} & \textbf{14.23} & \textbf{46.37} & 61.20 & \textbf{29.15} & \textbf{25.49} \\
        \midrule
        RLCR & 65.33 & 82.36 & 14.43 & 14.46 & 43.96 & 69.38 & 31.14 & 31.17 \\
        \rowcolor{calibrow} + CalibSFT & \textbf{67.20} & \textbf{86.32} & \textbf{12.77} & \textbf{8.59} & \textbf{45.71} & \textbf{71.31} & \textbf{23.24} & \textbf{18.20} \\
        \midrule
        ReDoubt & 65.59 & 81.05 & 15.48 & 16.94 & 43.71 & 69.88 & 25.07 & 23.98 \\
        \rowcolor{calibrow} + CalibSFT & \textbf{67.59} & \textbf{86.99} & \textbf{12.90} & \textbf{14.28} & \textbf{45.55} & \textbf{73.93} & \textbf{22.42} & \textbf{18.89} \\
        \midrule
        CoCA & 63.62 & 80.66 & 16.31 & 19.95 & 43.68 & 66.79 & 26.67 & 24.12 \\
        \rowcolor{calibrow} + CalibSFT & \textbf{65.07} & \textbf{85.63} & \textbf{16.06} & \textbf{18.51} & \textbf{45.55} & \textbf{71.22} & \textbf{22.94} & \textbf{18.89} \\
        \midrule
        DCPO & 64.68 & 78.21 & 14.69 & 12.16 & 43.99 & 68.13 & 32.85 & 32.41 \\
        \rowcolor{calibrow} + CalibSFT & \textbf{64.84} & \textbf{85.08} & \textbf{12.83} & \textbf{10.39} & \textbf{44.82} & \textbf{72.75} & \textbf{21.73} & \textbf{17.31} \\
        \midrule
        C3RL & 65.62 & 68.36 & 25.44 & 28.06 & 43.44 & 62.19 & 45.54 & 48.00 \\
        \rowcolor{calibrow} + CalibSFT & \textbf{66.17} & \textbf{79.70} & \textbf{15.91} & \textbf{14.35} & \textbf{44.60} & \textbf{63.03} & \textbf{23.89} & \textbf{17.38} \\
        \bottomrule
    \end{tabular}
    }
\end{table*}

%% file: tables/confidence_correlation.tex
\begin{table*}[t]
\centering
\caption{Pearson $r$ and Spearman $\rho$ correlation between accuracy and mean confidence across ID and OOD datasets.
Both correlations increase for all RL methods when incorporating CalibSFT, indicating that confidence can better reflect task-level accuracy.}
\label{tab:confidence_correlation}
\small
\setlength{\tabcolsep}{6pt}
\resizebox{0.95\textwidth}{!}{%
\begin{tabular}{ll*{5}{cc}}
\toprule
& & \multicolumn{2}{c}{RLCR} & \multicolumn{2}{c}{ReDoubt} & \multicolumn{2}{c}{CoCA} & \multicolumn{2}{c}{DCPO} & \multicolumn{2}{c}{C3RL} \\
\cmidrule(lr){3-4}\cmidrule(lr){5-6}\cmidrule(lr){7-8}\cmidrule(lr){9-10}\cmidrule(lr){11-12}
Dataset & Method & $r$ & $\rho$ & $r$ & $\rho$ & $r$ & $\rho$ & $r$ & $\rho$ & $r$ & $\rho$ \\
\midrule
 & Baseline & 0.768 & 0.619 & 0.829 & 0.762 & 0.750 & 0.643 & 0.669 & 0.619 & 0.635 & 0.595 \\
\rowcolor{calibrow} \cellcolor{white}\multirow{-2}{*}{ID} & +CalibSFT & \textbf{0.963} & \textbf{0.952} & \textbf{0.976} & \textbf{0.976} & \textbf{0.887} & \textbf{0.976} & \textbf{0.996} & \textbf{0.976} & \textbf{0.896} & \textbf{0.833} \\
\midrule
 & Baseline & 0.505 & 0.548 & 0.667 & 0.762 & 0.447 & 0.571 & 0.467 & 0.548 & 0.681 & 0.714 \\
\rowcolor{calibrow} \cellcolor{white}\multirow{-2}{*}{OOD} & +CalibSFT & \textbf{0.729} & \textbf{0.619} & \textbf{0.738} & \textbf{0.881} & \textbf{0.605} & \textbf{0.619} & \textbf{0.611} & \textbf{0.643} & \textbf{0.848} & \textbf{0.929} \\
\bottomrule
\end{tabular}}
\end{table*}

%% file: tables/format_calibsft.tex
\begin{table*}[t]
    \centering
    \caption{AUROC ($\uparrow$) and ECE ($\downarrow$) of RLCR on Qwen3-8B across confidence formats.
    ``Base'' denotes standard RLCR, and ``+Ours'' denotes RLCR initialized from CalibSFT.
    Across all four formats, CalibSFT consistently improves AUROC and reduces ECE across the 8 ID and 8 OOD benchmarks.}
    \label{tab:format_calibsft}
    \resizebox{0.95\textwidth}{!}{%
    \begin{tabular}{lcccccccc}
\toprule
& \multicolumn{4}{c}{ID} & \multicolumn{4}{c}{OOD} \\
\cmidrule(lr){2-5} \cmidrule(lr){6-9}
& \multicolumn{2}{c}{ECE $\downarrow$} & \multicolumn{2}{c}{AUROC $\uparrow$}
& \multicolumn{2}{c}{ECE $\downarrow$} & \multicolumn{2}{c}{AUROC $\uparrow$} \\
\cmidrule(lr){2-3} \cmidrule(lr){4-5} \cmidrule(lr){6-7} \cmidrule(lr){8-9}
Confidence Format & Base & +Ours & Base & +Ours & Base & +Ours & Base & +Ours \\
\midrule
Linguistic (low / medium / high) & 17.71 & \textbf{15.90} & 73.56 & \textbf{79.58} & 24.06 & \textbf{20.31} & 64.20 & \textbf{67.73} \\
Integer (0--9) & 14.69 & \textbf{14.06} & 79.00 & \textbf{84.14} & 26.31 & \textbf{16.70} & 67.00 & \textbf{69.36} \\
Probability $\{0.00, 0.05, \ldots, 1.00\}$ & 14.05 & \textbf{12.13} & 81.81 & \textbf{86.05} & 32.81 & \textbf{19.30} & 68.72 & \textbf{71.23} \\
\rowcolor{calibrow} Probability $\{0.00, 0.01, \ldots, 1.00\}$ & 14.05 & \textbf{8.59} & 81.81 & \textbf{86.32} & 32.81 & \textbf{18.20} & 68.72 & \textbf{71.31} \\
\bottomrule
\end{tabular}}
\end{table*}

%% file: tables/gemma_calibsft.tex
\begin{table*}[t]
    \centering
    \caption{AUROC ($\uparrow$) and ECE ($\downarrow$) of RLCR on Gemma4-E2B-Instruct across
    benchmarks of varying difficulty. CalibSFT noticeably improves calibration
    across all difficulty levels.}
    \label{tab:gemma_calibsft}
    \resizebox{0.95\textwidth}{!}{%
    \begin{tabular}{ccccccccc}
        \toprule
        & & \multicolumn{3}{c}{In-distribution} & \multicolumn{3}{c}{Out-of-distribution} & \\
        \cmidrule(lr){3-5} \cmidrule(lr){6-8}
        Metric & Method & Math500 & Olympiad & Minerva
        & DROP & HotpotQA & NQOpen & Average \\
        \midrule
         & RLCR
        & 7.84 & 17.69 & 30.34 & 19.92 & 26.41 & 38.25 & 23.41 \\
        \rowcolor{calibrow} \cellcolor{white}\multirow{-2}{*}{ECE $\downarrow$} & +CalibSFT
        & \textbf{5.34} & \textbf{7.76} & \textbf{12.74}
        & \textbf{13.30} & \textbf{13.04} & \textbf{22.77} & \textbf{12.49} \\
        \midrule
         & RLCR
        & 75.47 & 83.24 & 66.29 & 60.60 & 60.48 & 76.85 & 70.49 \\
        \rowcolor{calibrow} \cellcolor{white}\multirow{-2}{*}{AUROC $\uparrow$} & +CalibSFT
        & \textbf{80.34} & \textbf{87.35} & \textbf{70.75}
        & \textbf{65.45} & \textbf{62.66} & \textbf{77.13} & \textbf{73.95} \\
        \bottomrule
    \end{tabular}}
\end{table*}

%% file: tables/format_prompts.tex
\begin{table}[t]
    \centering
    \caption{System prompts for the three confidence formats.}
    \label{tab:format_prompts}
    \small
    \begin{tabularx}{\linewidth}{lX}
        \toprule
        Format & System prompt \\
        \midrule
        Probability &
        A conversation between User and Assistant. The user asks a question, and
        the Assistant solves it. The assistant first thinks about the reasoning
        process in the mind and analyzes its confidence about the solution and
        then provides the user with the final answer as well as its confidence
        level. The confidence level indicates how certain the Assistant is about
        its answer, expressed as a decimal between 0 and 1 with exactly two
        decimal places, enclosed within \texttt{<confidence> </confidence>} tags.
        The response must strictly follow this format: \texttt{<think>} reasoning
        process here \texttt{</think> <answer>} final short answer only
        \texttt{</answer> <confidence>} 0.xx \texttt{</confidence>}. The
        \texttt{<answer>} tag must contain only the final answer string needed for
        exact-match evaluation, not a full sentence, explanation, or reasoning. \\
        \midrule
        Integer &
        A conversation between User and Assistant. The user asks a question, and
        the Assistant solves it. The assistant first thinks about the reasoning
        process in the mind and analyzes its confidence about the solution and
        then provides the user with the final answer as well as its confidence
        score. The confidence score indicates how certain the Assistant is about
        its answer, expressed as a single integer from 0 to 9, enclosed within
        \texttt{<confidence> </confidence>} tags. A score of 0 means very low
        confidence and the answer is likely incorrect. A score of 9 means very
        high confidence and the answer is very likely correct. The response must
        strictly follow this format: \texttt{<think>} reasoning process here
        \texttt{</think> <answer>} final short answer only
        \texttt{</answer> <confidence>}X\texttt{</confidence>}, where X is one of
        0, 1, 2, 3, 4, 5, 6, 7, 8, 9. The \texttt{<answer>} tag must contain only
        the final answer string needed for exact-match evaluation, not a full
        sentence, explanation, or reasoning. \\
        \midrule
        Linguistic &
        A conversation between User and Assistant. The user asks a question, and
        the Assistant solves it. The assistant first thinks about the reasoning
        process in the mind and analyzes its confidence about the solution and
        then provides the user with the final answer as well as its confidence
        level. The confidence level indicates how certain the Assistant is about
        its answer, expressed as exactly one of low, medium, or high, enclosed
        within \texttt{<confidence> </confidence>} tags. Low means the answer is
        likely incorrect, medium means the answer is uncertain, and high means
        the answer is likely correct. The response must strictly follow this
        format: \texttt{<think>} reasoning process here
        \texttt{</think> <answer>} final short answer only
        \texttt{</answer> <confidence>}label\texttt{</confidence>}. The
        \texttt{<answer>} tag must contain only the final answer string needed for
        exact-match evaluation, not a full sentence, explanation, or reasoning. \\
        \bottomrule
    \end{tabularx}
\end{table}

%% file: tables/per_dataset_results.tex
\newcommand{\appendixdatasettable}[2]{%
\begin{subtable}[t]{0.32\textwidth}
\centering
\caption{#1}
\scriptsize
\setlength{\tabcolsep}{2pt}
\renewcommand\arraystretch{0.92}
\begin{tabular}{lcccc}
\toprule
Method & P@1 & AUROC & Brier & ECE \\
\midrule
#2
\bottomrule
\end{tabular}
\end{subtable}%
}

\begin{table*}[!t]
\centering
\caption{Per-dataset results on eight in-domain mathematical benchmarks
spanning a range of difficulty. Each entry reports Pass@1, AUROC, Brier
score, and ECE. CalibSFT initialization improves calibration and discrimination on most
datasets, with the largest gains on benchmarks where baseline RL is heavily
overconfident.}
\label{tab:appendix_id_results}
\appendixdatasettable{DeepScaleR-Eval}{
Base & 61.00 & 66.38 & 33.94 & 33.90 \\
RLVR & \textbf{70.26} & 67.94 & 26.85 & 26.26 \\
\rowcolor{calibrow} CalibSFT & 61.30 & \textbf{76.39} & \textbf{22.65} & \textbf{17.28} \\
\midrule
RLCR & 71.47 & 75.45 & 16.56 & 5.02 \\
\rowcolor{calibrow} + CalibSFT & \textbf{72.47} & \textbf{79.19} & \textbf{15.43} & \textbf{4.63} \\
\midrule
ReDoubt & 71.62 & 75.91 & 17.12 & 10.28 \\
\rowcolor{calibrow} + CalibSFT & \textbf{72.42} & \textbf{77.98} & \textbf{16.07} & \textbf{6.71} \\
\midrule
CoCA & 70.33 & 73.21 & 17.69 & 8.50 \\
\rowcolor{calibrow} + CalibSFT & \textbf{71.54} & \textbf{78.46} & \textbf{16.45} & \textbf{7.76} \\
\midrule
DCPO & \textbf{71.66} & 71.72 & 17.23 & 6.36 \\
\rowcolor{calibrow} + CalibSFT & 70.54 & \textbf{77.54} & \textbf{16.16} & \textbf{5.87} \\
\midrule
C3RL & 70.84 & 62.79 & 24.50 & 24.86 \\
\rowcolor{calibrow} + CalibSFT & \textbf{71.25} & \textbf{71.52} & \textbf{20.08} & \textbf{14.85} \\
}\hfill
\appendixdatasettable{MATH-500}{
Base & 84.50 & 65.22 & 14.18 & 13.27 \\
RLVR & \textbf{89.50} & 67.63 & \textbf{9.80} & \textbf{8.41} \\
\rowcolor{calibrow} CalibSFT & 84.00 & \textbf{81.51} & 16.49 & 14.06 \\
\midrule
RLCR & \textbf{90.80} & 76.45 & 7.42 & 7.53 \\
\rowcolor{calibrow} + CalibSFT & 90.35 & \textbf{89.86} & \textbf{6.46} & \textbf{3.89} \\
\midrule
ReDoubt & \textbf{90.10} & 75.41 & 9.79 & \textbf{15.06} \\
\rowcolor{calibrow} + CalibSFT & 89.70 & \textbf{89.44} & \textbf{9.59} & 16.99 \\
\midrule
CoCA & \textbf{90.10} & 73.02 & 10.57 & 18.22 \\
\rowcolor{calibrow} + CalibSFT & 90.05 & \textbf{87.05} & \textbf{9.08} & \textbf{15.69} \\
\midrule
DCPO & 89.65 & 70.59 & \textbf{7.67} & \textbf{6.52} \\
\rowcolor{calibrow} + CalibSFT & \textbf{90.40} & \textbf{84.55} & 7.71 & 7.83 \\
\midrule
C3RL & \textbf{89.85} & 60.68 & \textbf{8.97} & \textbf{9.30} \\
\rowcolor{calibrow} + CalibSFT & 88.05 & \textbf{84.35} & 10.47 & 9.71 \\
}\hfill
\appendixdatasettable{MinervaMath}{
Base & 47.52 & 60.84 & 46.83 & 47.65 \\
RLVR & \textbf{52.85} & 60.81 & 43.34 & 43.51 \\
\rowcolor{calibrow} CalibSFT & 44.94 & \textbf{66.31} & \textbf{25.37} & \textbf{17.42} \\
\midrule
RLCR & 51.75 & 61.57 & 29.17 & 23.05 \\
\rowcolor{calibrow} + CalibSFT & \textbf{51.93} & \textbf{73.00} & \textbf{22.59} & \textbf{12.00} \\
\midrule
ReDoubt & \textbf{52.30} & 60.64 & 25.36 & 10.67 \\
\rowcolor{calibrow} + CalibSFT & 51.38 & \textbf{74.47} & \textbf{20.96} & \textbf{9.35} \\
\midrule
CoCA & 49.63 & 64.28 & 26.82 & 18.83 \\
\rowcolor{calibrow} + CalibSFT & \textbf{50.55} & \textbf{76.38} & \textbf{21.78} & \textbf{13.77} \\
\midrule
DCPO & \textbf{50.64} & 63.57 & 31.72 & 29.05 \\
\rowcolor{calibrow} + CalibSFT & 50.55 & \textbf{76.00} & \textbf{20.53} & \textbf{8.83} \\
\midrule
C3RL & 50.37 & 57.50 & 45.46 & 46.34 \\
\rowcolor{calibrow} + CalibSFT & \textbf{51.01} & \textbf{73.31} & \textbf{21.79} & \textbf{13.06} \\
}

\medskip
\appendixdatasettable{OlympiadBench}{
Base & \textbf{61.42} & 74.08 & 32.51 & 33.33 \\
RLVR & 61.39 & 73.08 & 34.64 & 34.90 \\
\rowcolor{calibrow} CalibSFT & 60.50 & \textbf{85.75} & \textbf{16.72} & \textbf{11.61} \\
\midrule
RLCR & 63.17 & 83.46 & 16.73 & \textbf{8.48} \\
\rowcolor{calibrow} + CalibSFT & \textbf{63.28} & \textbf{89.00} & \textbf{14.60} & 9.25 \\
\midrule
ReDoubt & 62.80 & 81.10 & 17.14 & \textbf{7.73} \\
\rowcolor{calibrow} + CalibSFT & \textbf{63.58} & \textbf{90.10} & \textbf{14.29} & 14.69 \\
\midrule
CoCA & \textbf{61.57} & 81.47 & 18.37 & \textbf{14.56} \\
\rowcolor{calibrow} + CalibSFT & \textbf{61.57} & \textbf{88.17} & \textbf{17.05} & 17.07 \\
\midrule
DCPO & 62.24 & 79.80 & 18.80 & 12.97 \\
\rowcolor{calibrow} + CalibSFT & \textbf{62.91} & \textbf{87.43} & \textbf{14.76} & \textbf{8.93} \\
\midrule
C3RL & 63.06 & 64.34 & 30.64 & 32.22 \\
\rowcolor{calibrow} + CalibSFT & \textbf{63.61} & \textbf{80.73} & \textbf{18.05} & \textbf{15.26} \\
}\hfill
\appendixdatasettable{GSM8K}{
Base & 94.39 & 67.91 & 5.16 & 3.35 \\
RLVR & \textbf{94.71} & \textbf{68.09} & \textbf{4.99} & \textbf{2.98} \\
\rowcolor{calibrow} CalibSFT & 94.60 & 66.46 & 25.93 & 30.61 \\
\midrule
RLCR & \textbf{94.86} & 78.70 & 4.75 & 5.64 \\
\rowcolor{calibrow} + CalibSFT & 94.77 & \textbf{82.24} & \textbf{4.65} & \textbf{2.81} \\
\midrule
ReDoubt & \textbf{95.00} & 78.24 & \textbf{7.07} & \textbf{16.31} \\
\rowcolor{calibrow} + CalibSFT & 94.77 & \textbf{80.01} & 7.48 & 17.69 \\
\midrule
CoCA & \textbf{94.66} & 68.97 & 8.91 & 21.10 \\
\rowcolor{calibrow} + CalibSFT & 94.64 & \textbf{71.42} & \textbf{8.25} & \textbf{18.18} \\
\midrule
DCPO & \textbf{94.58} & 68.17 & \textbf{5.44} & 9.17 \\
\rowcolor{calibrow} + CalibSFT & 94.50 & \textbf{73.10} & 5.82 & \textbf{7.54} \\
\midrule
C3RL & \textbf{94.43} & 57.33 & \textbf{5.28} & 5.18 \\
\rowcolor{calibrow} + CalibSFT & 94.37 & \textbf{64.41} & 5.30 & \textbf{5.16} \\
}\hfill
\appendixdatasettable{AIME2024}{
Base & 48.23 & 81.93 & 39.61 & 42.50 \\
RLVR & \textbf{52.92} & 82.05 & 39.49 & 41.15 \\
\rowcolor{calibrow} CalibSFT & 47.08 & \textbf{89.04} & \textbf{10.35} & \textbf{7.73} \\
\midrule
RLCR & 59.48 & \textbf{93.59} & 13.17 & 19.70 \\
\rowcolor{calibrow} + CalibSFT & \textbf{64.17} & 92.42 & \textbf{10.97} & \textbf{9.79} \\
\midrule
ReDoubt & 55.62 & 91.94 & 15.40 & 22.77 \\
\rowcolor{calibrow} + CalibSFT & \textbf{61.56} & \textbf{93.81} & \textbf{11.89} & \textbf{14.37} \\
\midrule
CoCA & 53.85 & 92.22 & \textbf{16.27} & \textbf{21.37} \\
\rowcolor{calibrow} + CalibSFT & \textbf{58.54} & \textbf{95.41} & 16.34 & 23.56 \\
\midrule
DCPO & 57.92 & 88.84 & 12.32 & \textbf{8.47} \\
\rowcolor{calibrow} + CalibSFT & \textbf{58.33} & \textbf{93.89} & \textbf{11.02} & 11.85 \\
\midrule
C3RL & 59.69 & 80.93 & 25.25 & 30.19 \\
\rowcolor{calibrow} + CalibSFT & \textbf{64.27} & \textbf{85.24} & \textbf{14.04} & \textbf{12.99} \\
}

\medskip
\appendixdatasettable{AIME2025}{
Base & 35.00 & 82.75 & 48.25 & 52.81 \\
RLVR & \textbf{43.44} & 84.95 & 45.87 & 49.00 \\
\rowcolor{calibrow} CalibSFT & 35.83 & \textbf{86.49} & \textbf{9.61} & \textbf{6.59} \\
\midrule
RLCR & 43.33 & \textbf{95.53} & \textbf{13.60} & 23.06 \\
\rowcolor{calibrow} + CalibSFT & \textbf{48.02} & 92.40 & 13.87 & \textbf{14.28} \\
\midrule
ReDoubt & 44.90 & 94.22 & 15.19 & 26.51 \\
\rowcolor{calibrow} + CalibSFT & \textbf{51.77} & \textbf{96.32} & \textbf{11.28} & \textbf{16.79} \\
\midrule
CoCA & 40.31 & \textbf{95.39} & \textbf{17.04} & 29.31 \\
\rowcolor{calibrow} + CalibSFT & \textbf{45.00} & 94.21 & 19.99 & \textbf{26.43} \\
\midrule
DCPO & 41.15 & 91.19 & \textbf{13.33} & \textbf{14.30} \\
\rowcolor{calibrow} + CalibSFT & \textbf{41.25} & \textbf{95.07} & 13.81 & 17.73 \\
\midrule
C3RL & \textbf{43.96} & 79.01 & 36.25 & 42.69 \\
\rowcolor{calibrow} + CalibSFT & 43.12 & \textbf{86.74} & \textbf{21.86} & \textbf{25.49} \\
}\hfill
\appendixdatasettable{AIME2026}{
Base & 42.50 & 76.08 & 45.13 & 48.07 \\
RLVR & \textbf{48.23} & 80.70 & 43.82 & 45.92 \\
\rowcolor{calibrow} CalibSFT & 39.79 & \textbf{87.44} & \textbf{11.26} & \textbf{8.52} \\
\midrule
RLCR & 47.81 & \textbf{94.10} & 14.08 & 23.19 \\
\rowcolor{calibrow} + CalibSFT & \textbf{52.60} & 92.48 & \textbf{13.57} & \textbf{12.04} \\
\midrule
ReDoubt & 52.40 & 90.93 & 16.78 & 26.20 \\
\rowcolor{calibrow} + CalibSFT & \textbf{55.52} & \textbf{93.82} & \textbf{11.66} & \textbf{17.69} \\
\midrule
CoCA & 48.54 & \textbf{96.74} & \textbf{14.84} & 27.72 \\
\rowcolor{calibrow} + CalibSFT & \textbf{48.65} & 93.96 & 19.52 & \textbf{25.63} \\
\midrule
DCPO & 49.58 & 91.78 & \textbf{11.04} & \textbf{10.41} \\
\rowcolor{calibrow} + CalibSFT & \textbf{50.21} & \textbf{93.07} & 12.87 & 14.52 \\
\midrule
C3RL & 52.81 & 84.29 & 27.21 & 33.72 \\
\rowcolor{calibrow} + CalibSFT & \textbf{53.65} & \textbf{91.31} & \textbf{15.69} & \textbf{18.30} \\
}\hfill
\appendixdatasettable{Overall (ID)}{
Base & 59.32 & 71.90 & 33.20 & 34.36 \\
RLVR & \textbf{64.16} & 73.16 & 31.10 & 31.52 \\
\rowcolor{calibrow} CalibSFT & 58.51 & \textbf{79.93} & \textbf{17.30} & \textbf{14.23} \\
\midrule
RLCR & 65.33 & 82.36 & 14.43 & 14.46 \\
\rowcolor{calibrow} + CalibSFT & \textbf{67.20} & \textbf{86.32} & \textbf{12.77} & \textbf{8.59} \\
\midrule
ReDoubt & 65.59 & 81.05 & 15.48 & 16.94 \\
\rowcolor{calibrow} + CalibSFT & \textbf{67.59} & \textbf{86.99} & \textbf{12.90} & \textbf{14.28} \\
\midrule
CoCA & 63.62 & 80.66 & 16.31 & 19.95 \\
\rowcolor{calibrow} + CalibSFT & \textbf{65.07} & \textbf{85.63} & \textbf{16.06} & \textbf{18.51} \\
\midrule
DCPO & 64.68 & 78.21 & 14.69 & 12.16 \\
\rowcolor{calibrow} + CalibSFT & \textbf{64.84} & \textbf{85.08} & \textbf{12.83} & \textbf{10.39} \\
\midrule
C3RL & 65.62 & 68.36 & 25.44 & 28.06 \\
\rowcolor{calibrow} + CalibSFT & \textbf{66.17} & \textbf{79.70} & \textbf{15.91} & \textbf{14.35} \\
}
\end{table*}

\begin{table*}[!t]
\centering
\caption{Per-dataset results on eight out-of-domain general reasoning
benchmarks spanning a range of difficulty. Each entry reports Pass@1, AUROC,
Brier score, and ECE. CalibSFT initialization improves calibration and
discrimination on most datasets, with the largest gains on benchmarks where
baseline RL is heavily overconfident.}
\label{tab:appendix_ood_results}
\appendixdatasettable{HotpotQA}{
Base & 64.68 & \textbf{57.68} & 32.82 & 32.16 \\
RLVR & 64.58 & 54.16 & 33.10 & 32.17 \\
\rowcolor{calibrow} CalibSFT & \textbf{66.10} & 55.08 & \textbf{30.84} & \textbf{22.43} \\
\midrule
RLCR & 65.80 & 59.10 & 25.97 & 19.77 \\
\rowcolor{calibrow} + CalibSFT & \textbf{66.27} & \textbf{59.23} & \textbf{24.25} & \textbf{10.85} \\
\midrule
ReDoubt & \textbf{64.83} & 60.64 & 23.52 & 12.11 \\
\rowcolor{calibrow} + CalibSFT & 64.68 & \textbf{63.08} & \textbf{21.70} & \textbf{1.37} \\
\midrule
CoCA & 64.90 & 56.43 & \textbf{22.77} & \textbf{6.02} \\
\rowcolor{calibrow} + CalibSFT & \textbf{66.10} & \textbf{61.57} & 22.92 & 11.39 \\
\midrule
DCPO & \textbf{66.15} & 57.88 & 25.33 & 18.77 \\
\rowcolor{calibrow} + CalibSFT & 65.97 & \textbf{60.70} & \textbf{24.43} & \textbf{14.92} \\
\midrule
C3RL & 65.38 & 52.47 & 33.39 & 33.12 \\
\rowcolor{calibrow} + CalibSFT & \textbf{66.50} & \textbf{58.94} & \textbf{28.45} & \textbf{19.78} \\
}\hfill
\appendixdatasettable{TriviaQA}{
Base & 55.24 & 75.80 & 35.68 & 36.93 \\
RLVR & 54.97 & \textbf{76.01} & 40.59 & 41.40 \\
\rowcolor{calibrow} CalibSFT & \textbf{55.53} & 64.10 & \textbf{26.74} & \textbf{15.17} \\
\midrule
RLCR & 55.24 & \textbf{81.66} & 23.84 & 23.30 \\
\rowcolor{calibrow} + CalibSFT & \textbf{55.25} & 81.53 & \textbf{18.89} & \textbf{10.71} \\
\midrule
ReDoubt & 55.69 & 82.96 & \textbf{18.67} & \textbf{13.58} \\
\rowcolor{calibrow} + CalibSFT & \textbf{56.14} & \textbf{83.92} & 19.27 & 15.58 \\
\midrule
CoCA & 54.77 & 77.62 & 21.63 & \textbf{12.97} \\
\rowcolor{calibrow} + CalibSFT & \textbf{54.89} & \textbf{81.85} & \textbf{20.98} & 13.16 \\
\midrule
DCPO & 54.55 & 79.86 & 25.93 & 23.91 \\
\rowcolor{calibrow} + CalibSFT & \textbf{55.12} & \textbf{82.56} & \textbf{17.59} & \textbf{8.25} \\
\midrule
C3RL & 53.76 & \textbf{73.01} & 35.61 & 38.53 \\
\rowcolor{calibrow} + CalibSFT & \textbf{54.11} & 71.36 & \textbf{22.75} & \textbf{12.50} \\
}\hfill
\appendixdatasettable{DROP}{
Base & 80.36 & 59.13 & 18.22 & 17.03 \\
RLVR & \textbf{83.53} & \textbf{59.32} & \textbf{15.34} & \textbf{13.31} \\
\rowcolor{calibrow} CalibSFT & 79.76 & 58.30 & 33.30 & 32.64 \\
\midrule
RLCR & 81.44 & 63.34 & \textbf{14.72} & \textbf{5.43} \\
\rowcolor{calibrow} + CalibSFT & \textbf{82.32} & \textbf{74.44} & 14.99 & 10.75 \\
\midrule
ReDoubt & \textbf{83.63} & 63.27 & \textbf{14.76} & \textbf{10.80} \\
\rowcolor{calibrow} + CalibSFT & 82.53 & \textbf{75.90} & 15.60 & 17.36 \\
\midrule
CoCA & 82.47 & 61.58 & \textbf{15.03} & \textbf{11.56} \\
\rowcolor{calibrow} + CalibSFT & \textbf{82.50} & \textbf{73.91} & 17.95 & 21.99 \\
\midrule
DCPO & \textbf{83.45} & 59.34 & \textbf{13.25} & \textbf{2.30} \\
\rowcolor{calibrow} + CalibSFT & 80.42 & \textbf{75.92} & 16.54 & 13.64 \\
\midrule
C3RL & \textbf{82.65} & 54.56 & \textbf{16.37} & \textbf{16.07} \\
\rowcolor{calibrow} + CalibSFT & 77.44 & \textbf{64.78} & 23.65 & 22.03 \\
}

\medskip
\appendixdatasettable{MuSiQue}{
Base & 45.84 & \textbf{62.90} & 41.91 & 42.33 \\
RLVR & 47.18 & 58.12 & 46.19 & 46.75 \\
\rowcolor{calibrow} CalibSFT & \textbf{48.09} & 58.78 & \textbf{28.43} & \textbf{16.30} \\
\midrule
RLCR & 47.77 & 66.21 & 25.68 & 15.35 \\
\rowcolor{calibrow} + CalibSFT & \textbf{49.97} & \textbf{67.87} & \textbf{23.22} & \textbf{6.13} \\
\midrule
ReDoubt & 46.90 & 66.37 & 24.07 & 11.21 \\
\rowcolor{calibrow} + CalibSFT & \textbf{48.52} & \textbf{70.04} & \textbf{22.42} & \textbf{5.52} \\
\midrule
CoCA & 46.69 & \textbf{67.03} & 23.69 & 10.18 \\
\rowcolor{calibrow} + CalibSFT & \textbf{48.76} & 66.88 & \textbf{23.56} & \textbf{7.72} \\
\midrule
DCPO & 47.34 & 65.48 & 28.28 & 20.53 \\
\rowcolor{calibrow} + CalibSFT & \textbf{48.96} & \textbf{70.49} & \textbf{22.75} & \textbf{8.55} \\
\midrule
C3RL & 45.55 & 63.39 & 39.88 & 40.04 \\
\rowcolor{calibrow} + CalibSFT & \textbf{47.47} & \textbf{66.39} & \textbf{24.46} & \textbf{11.87} \\
}\hfill
\appendixdatasettable{LiveBenchReasoning}{
Base & 52.88 & 61.81 & 40.53 & 40.59 \\
RLVR & 29.25 & 60.32 & 61.75 & 64.59 \\
\rowcolor{calibrow} CalibSFT & \textbf{55.75} & \textbf{69.92} & \textbf{34.64} & \textbf{35.20} \\
\midrule
RLCR & 31.75 & 68.00 & 28.49 & 29.09 \\
\rowcolor{calibrow} + CalibSFT & \textbf{42.25} & \textbf{68.41} & \textbf{21.94} & \textbf{7.75} \\
\midrule
ReDoubt & 25.75 & 68.44 & 23.51 & 24.86 \\
\rowcolor{calibrow} + CalibSFT & \textbf{42.25} & \textbf{80.49} & \textbf{19.95} & \textbf{12.38} \\
\midrule
CoCA & 32.12 & 66.75 & 25.79 & 23.64 \\
\rowcolor{calibrow} + CalibSFT & \textbf{41.50} & \textbf{68.03} & \textbf{23.24} & \textbf{7.96} \\
\midrule
DCPO & 31.13 & 67.14 & 32.10 & 33.12 \\
\rowcolor{calibrow} + CalibSFT & \textbf{37.75} & \textbf{71.77} & \textbf{21.12} & \textbf{8.83} \\
\midrule
C3RL & 33.00 & \textbf{57.18} & 53.55 & 55.46 \\
\rowcolor{calibrow} + CalibSFT & \textbf{44.62} & 53.33 & \textbf{28.11} & \textbf{22.78} \\
}\hfill
\appendixdatasettable{NQOpen}{
Base & 21.12 & 66.31 & 66.04 & 70.59 \\
RLVR & \textbf{24.13} & \textbf{67.07} & 68.92 & 71.55 \\
\rowcolor{calibrow} CalibSFT & 21.57 & 59.58 & \textbf{25.69} & \textbf{26.36} \\
\midrule
RLCR & 23.32 & 70.48 & 43.84 & 52.43 \\
\rowcolor{calibrow} + CalibSFT & \textbf{23.80} & \textbf{72.24} & \textbf{26.44} & \textbf{30.89} \\
\midrule
ReDoubt & \textbf{24.67} & 70.28 & 32.15 & 39.19 \\
\rowcolor{calibrow} + CalibSFT & 24.10 & \textbf{70.60} & \textbf{26.80} & \textbf{31.99} \\
\midrule
CoCA & 22.76 & 66.59 & 35.28 & 43.26 \\
\rowcolor{calibrow} + CalibSFT & \textbf{23.72} & \textbf{70.77} & \textbf{24.86} & \textbf{28.69} \\
\midrule
DCPO & \textbf{23.69} & 69.80 & 46.27 & 53.69 \\
\rowcolor{calibrow} + CalibSFT & 23.58 & \textbf{73.13} & \textbf{22.65} & \textbf{25.97} \\
\midrule
C3RL & \textbf{22.43} & \textbf{64.77} & 63.63 & 68.49 \\
\rowcolor{calibrow} + CalibSFT & 22.08 & 61.18 & \textbf{22.25} & \textbf{17.54} \\
}

\medskip
\appendixdatasettable{PopQA}{
Base & 20.85 & \textbf{74.94} & 57.89 & 64.10 \\
RLVR & \textbf{21.37} & 72.78 & 68.14 & 71.96 \\
\rowcolor{calibrow} CalibSFT & 21.06 & 64.33 & \textbf{24.14} & \textbf{25.70} \\
\midrule
RLCR & \textbf{21.29} & \textbf{76.88} & 37.71 & 47.90 \\
\rowcolor{calibrow} + CalibSFT & 21.16 & 76.55 & \textbf{24.76} & \textbf{32.22} \\
\midrule
ReDoubt & \textbf{21.37} & 76.83 & 26.82 & 36.23 \\
\rowcolor{calibrow} + CalibSFT & 21.25 & \textbf{77.87} & \textbf{24.54} & \textbf{32.32} \\
\midrule
CoCA & \textbf{21.06} & 73.92 & 30.84 & 39.88 \\
\rowcolor{calibrow} + CalibSFT & 20.88 & \textbf{76.99} & \textbf{24.12} & \textbf{31.76} \\
\midrule
DCPO & 20.99 & \textbf{77.02} & 41.10 & 49.96 \\
\rowcolor{calibrow} + CalibSFT & \textbf{21.18} & 76.52 & \textbf{21.29} & \textbf{28.02} \\
\midrule
C3RL & \textbf{20.81} & \textbf{71.98} & 54.60 & 61.92 \\
\rowcolor{calibrow} + CalibSFT & 20.66 & 66.98 & \textbf{16.56} & \textbf{12.95} \\
}\hfill
\appendixdatasettable{WebQuestions}{
Base & 21.88 & \textbf{67.10} & 68.16 & 71.80 \\
RLVR & \textbf{26.02} & 65.76 & 68.29 & 70.41 \\
\rowcolor{calibrow} CalibSFT & 23.11 & 59.54 & \textbf{29.45} & \textbf{30.13} \\
\midrule
RLCR & \textbf{25.07} & 69.40 & 48.86 & 56.11 \\
\rowcolor{calibrow} + CalibSFT & 24.63 & \textbf{70.20} & \textbf{31.41} & \textbf{36.33} \\
\midrule
ReDoubt & \textbf{26.84} & \textbf{70.22} & 37.08 & 43.81 \\
\rowcolor{calibrow} + CalibSFT & 24.91 & 69.57 & \textbf{29.11} & \textbf{34.59} \\
\midrule
CoCA & 24.65 & 64.40 & 38.35 & 45.46 \\
\rowcolor{calibrow} + CalibSFT & \textbf{26.06} & \textbf{69.71} & \textbf{25.93} & \textbf{28.47} \\
\midrule
DCPO & 24.59 & 68.48 & 50.52 & 57.04 \\
\rowcolor{calibrow} + CalibSFT & \textbf{25.61} & \textbf{70.87} & \textbf{27.45} & \textbf{30.32} \\
\midrule
C3RL & \textbf{23.94} & 60.19 & 67.27 & 70.35 \\
\rowcolor{calibrow} + CalibSFT & \textbf{23.94} & \textbf{61.29} & \textbf{24.87} & \textbf{19.55} \\
}\hfill
\appendixdatasettable{Overall (OOD)}{
Base & 45.36 & \textbf{65.71} & 45.15 & 46.94 \\
RLVR & 43.88 & 64.19 & 50.29 & 51.52 \\
\rowcolor{calibrow} CalibSFT & \textbf{46.37} & 61.20 & \textbf{29.15} & \textbf{25.49} \\
\midrule
RLCR & 43.96 & 69.38 & 31.14 & 31.17 \\
\rowcolor{calibrow} + CalibSFT & \textbf{45.71} & \textbf{71.31} & \textbf{23.24} & \textbf{18.20} \\
\midrule
ReDoubt & 43.71 & 69.88 & 25.07 & 23.98 \\
\rowcolor{calibrow} + CalibSFT & \textbf{45.55} & \textbf{73.93} & \textbf{22.42} & \textbf{18.89} \\
\midrule
CoCA & 43.68 & 66.79 & 26.67 & 24.12 \\
\rowcolor{calibrow} + CalibSFT & \textbf{45.55} & \textbf{71.22} & \textbf{22.94} & \textbf{18.89} \\
\midrule
DCPO & 43.99 & 68.13 & 32.85 & 32.41 \\
\rowcolor{calibrow} + CalibSFT & \textbf{44.82} & \textbf{72.75} & \textbf{21.73} & \textbf{17.31} \\
\midrule
C3RL & 43.44 & 62.19 & 45.54 & 48.00 \\
\rowcolor{calibrow} + CalibSFT & \textbf{44.60} & \textbf{63.03} & \textbf{23.89} & \textbf{17.38} \\
}
\end{table*}

%% file: tables/seed_stability.tex
\begin{table}[t]
    \centering
    \caption{Performance stability across random seeds on Qwen3-8B.
    Each entry reports the mean and standard deviation
    over 3 random seeds (43, 44, and 45), following the evaluation protocol of Table~\ref{tab:main_results}.
    The variation remains small across all metrics, and CalibSFT consistently outperforms RL baselines while maintaining task accuracy.}
    \label{tab:seed_stability}
    \resizebox{\linewidth}{!}{%
    \renewcommand\arraystretch{1.0}
    \begin{tabular}{lcccccccc}
        \toprule
        & \multicolumn{4}{c}{In-distribution} & \multicolumn{4}{c}{Out-of-distribution} \\
        \cmidrule(lr){2-5} \cmidrule(lr){6-9}
        Method & Pass@1 $\uparrow$ & AUROC $\uparrow$ & Brier $\downarrow$
        & ECE $\downarrow$
        & Pass@1 $\uparrow$ & AUROC $\uparrow$ & Brier $\downarrow$
        & ECE $\downarrow$ \\
        \midrule
        RLCR & 65.50$_{\pm0.66}$ & 82.08$_{\pm0.38}$ & 14.64$_{\pm0.29}$ & 14.26$_{\pm0.29}$ & 43.53$_{\pm0.61}$ & 69.05$_{\pm0.47}$ & 31.67$_{\pm0.75}$ & 31.99$_{\pm1.15}$ \\
        \rowcolor{calibrow} + CalibSFT & \textbf{67.15}$_{\pm0.92}$ & \textbf{86.29}$_{\pm0.19}$ & \textbf{12.98}$_{\pm0.47}$ & \textbf{9.36}$_{\pm1.51}$ & \textbf{45.69}$_{\pm0.12}$ & \textbf{71.77}$_{\pm0.65}$ & \textbf{22.70}$_{\pm0.77}$ & \textbf{17.75}$_{\pm0.64}$ \\
        \midrule
        DCPO & 64.69$_{\pm0.21}$ & 78.40$_{\pm0.41}$ & 14.63$_{\pm0.19}$ & 11.78$_{\pm0.33}$ & 43.77$_{\pm0.39}$ & 68.37$_{\pm0.54}$ & 33.22$_{\pm1.15}$ & 32.69$_{\pm1.09}$ \\
        \rowcolor{calibrow} + CalibSFT & \textbf{64.90}$_{\pm0.19}$ & \textbf{85.22}$_{\pm0.15}$ & \textbf{12.45}$_{\pm0.29}$ & \textbf{10.15}$_{\pm0.60}$ & \textbf{44.89}$_{\pm0.12}$ & \textbf{72.91}$_{\pm0.42}$ & \textbf{21.60}$_{\pm0.58}$ & \textbf{17.39}$_{\pm0.36}$ \\

        \bottomrule
    \end{tabular}
    }
\end{table}

%% file: tables/additional_metrics.tex
\begin{table*}[t]
    \centering
    \caption{Calibration performance on additional metrics, on the in-distribution
    and out-of-distribution splits and under the protocol of Table~\ref{tab:main_results}. PCE is
    the part of ECE contributed by the overconfident bins, Cap.\ Brier scores each question's mean
    confidence against its rollout success rate, E-AURC is the excess of the risk--coverage area
    over a perfect ranking, and RCE measures the deviation from a monotone correspondence between
    confidence and expected correctness.
 Across all baseline RL methods, incorporating CalibSFT reduces overconfidence and misranking errors (lower PCE, E-AURC, and RCE) while improving question-level confidence alignment (lower Cap.~Brier).}

    \label{tab:additional_metrics}
    \resizebox{0.98\textwidth}{!}{%
    \renewcommand\arraystretch{1.0}
    \begin{tabular}{lcccccccc}
        \toprule
        & \multicolumn{4}{c}{In-distribution} & \multicolumn{4}{c}{Out-of-distribution} \\
        \cmidrule(lr){2-5} \cmidrule(lr){6-9}
        Method & PCE $\downarrow$ & Cap.\ Brier $\downarrow$ & E-AURC $\downarrow$ & RCE $\downarrow$ & PCE $\downarrow$ & Cap.\ Brier $\downarrow$ & E-AURC $\downarrow$ & RCE $\downarrow$ \\
        \midrule
        RLCR & 9.71 & 10.63 & 7.75 & 13.50 & 30.99 & 27.53 & 19.74 & 11.45 \\
        \rowcolor{calibrow} + CalibSFT & \textbf{7.30} & \textbf{8.66} & \textbf{5.61} & \textbf{5.30} & \textbf{14.77} & \textbf{18.72} & \textbf{17.14} & \textbf{4.93} \\
        \midrule
        ReDoubt & 6.71 & 11.98 & 8.02 & 15.23 & 22.51 & 21.71 & 18.34 & 14.77 \\
        \rowcolor{calibrow} + CalibSFT & \textbf{5.54} & \textbf{9.33} & \textbf{4.77} & \textbf{7.69} & \textbf{15.58} & \textbf{18.80} & \textbf{15.85} & \textbf{4.70} \\
        \midrule
        CoCA & \textbf{10.96} & 12.82 & 8.88 & 19.49 & 22.64 & 22.84 & 21.52 & 15.13 \\
        \rowcolor{calibrow} + CalibSFT & 11.40 & \textbf{12.24} & \textbf{5.50} & \textbf{8.58} & \textbf{13.39} & \textbf{19.41} & \textbf{17.82} & \textbf{5.79} \\
        \midrule
        DCPO & 10.04 & 10.95 & 10.10 & 18.79 & 31.97 & 28.95 & 19.62 & 15.40 \\
        \rowcolor{calibrow} + CalibSFT & \textbf{8.09} & \textbf{9.00} & \textbf{5.92} & \textbf{9.21} & \textbf{12.95} & \textbf{17.54} & \textbf{16.23} & \textbf{6.18} \\
        \midrule
        C3RL & 28.01 & 20.75 & 16.11 & 33.15 & 47.96 & 41.92 & 25.22 & 27.57 \\
        \rowcolor{calibrow} + CalibSFT & \textbf{10.72} & \textbf{10.82} & \textbf{8.76} & \textbf{17.24} & \textbf{8.80} & \textbf{18.46} & \textbf{21.27} & \textbf{19.08} \\
        \bottomrule
    \end{tabular}
    }
\end{table*}

%% file: tables/diversity_metrics.tex
\begin{table*}[t]
    \centering
    \caption{Confidence diversity on the in-distribution and out-of-distribution splits, under the
    protocol of Table~\ref{tab:main_results}. Support counts the unique confidence values a model
    produces, Top-1 and Top-3 are the share of responses taking its most frequent one and three
    values, and Effective Support (Eff.\ Supp.) is the exponential of the entropy of the value distribution. Baseline RL methods collapse into narrow confidence supports with heavy concentration on top values,
     whereas CalibSFT broadens both support size and effective support for all objectives.}

    \label{tab:diversity_metrics}
    \resizebox{0.98\textwidth}{!}{%
    \renewcommand\arraystretch{1.0}
    \begin{tabular}{lcccccccc}
        \toprule
        & \multicolumn{4}{c}{In-distribution} & \multicolumn{4}{c}{Out-of-distribution} \\
        \cmidrule(lr){2-5} \cmidrule(lr){6-9}
        Method & Support $\uparrow$ & Top-1 $\downarrow$ & Top-3 $\downarrow$ & Eff.\ Supp.\ $\uparrow$ & Support $\uparrow$ & Top-1 $\downarrow$ & Top-3 $\downarrow$ & Eff.\ Supp.\ $\uparrow$ \\
        \midrule
        RLCR & 15.75 & 33.30 & 67.74 & 6.90 & 18.62 & 33.18 & 66.22 & 7.03 \\
        \rowcolor{calibrow} + CalibSFT & \textbf{86.88} & \textbf{12.59} & \textbf{20.57} & \textbf{49.42} & \textbf{94.50} & \textbf{3.73} & \textbf{9.90} & \textbf{72.83} \\
        \midrule
        ReDoubt & 7.62 & 43.17 & 92.47 & 3.50 & 10.88 & 44.84 & 86.31 & 4.20 \\
        \rowcolor{calibrow} + CalibSFT & \textbf{58.00} & \textbf{9.43} & \textbf{23.64} & \textbf{32.03} & \textbf{59.50} & \textbf{7.83} & \textbf{20.70} & \textbf{31.62} \\
        \midrule
        CoCA & 12.75 & 55.50 & 77.38 & 4.85 & 13.62 & 46.41 & 67.06 & 6.06 \\
        \rowcolor{calibrow} + CalibSFT & \textbf{47.62} & \textbf{13.62} & \textbf{30.12} & \textbf{26.22} & \textbf{52.12} & \textbf{10.71} & \textbf{26.88} & \textbf{24.69} \\
        \midrule
        DCPO & 15.75 & 59.32 & 80.94 & 4.40 & 18.62 & 43.97 & 73.64 & 6.06 \\
        \rowcolor{calibrow} + CalibSFT & \textbf{72.75} & \textbf{21.54} & \textbf{41.93} & \textbf{27.14} & \textbf{84.25} & \textbf{5.78} & \textbf{15.07} & \textbf{53.93} \\
        \midrule
        C3RL & 8.50 & 83.23 & 96.46 & 1.97 & 9.38 & 71.04 & 91.77 & 2.79 \\
        \rowcolor{calibrow} + CalibSFT & \textbf{10.00} & \textbf{46.09} & \textbf{70.97} & \textbf{5.68} & \textbf{10.00} & \textbf{36.37} & \textbf{65.54} & \textbf{6.46} \\
        \bottomrule
    \end{tabular}
    }
\end{table*}

%% file: tables/training_target_distribution.tex
\begin{table}[t]
    \centering
    \caption{Distribution (\%) of training targets across ten confidence intervals.
    Target scores derived from prior methods (UaIT, SaySelf, LoVeC) are concentrated or skewed,
    whereas CalibSFT provides a balanced, near-uniform distribution across the entire range.}
    \label{tab:training_target_distribution}
    \footnotesize
    \setlength{\tabcolsep}{3pt}
    \resizebox{\linewidth}{!}{%
    \begin{tabular}{lrrrrrrrrrr}
        \toprule
        Target & $[0.0, 0.1]$ & $(0.1, 0.2]$ & $(0.2, 0.3]$ & $(0.3, 0.4]$ & $(0.4, 0.5]$
        & $(0.5, 0.6]$ & $(0.6, 0.7]$ & $(0.7, 0.8]$ & $(0.8, 0.9]$ & $(0.9, 1.0]$ \\
        \midrule
        CSFT & 11.11 & 11.11 & 11.11 & 11.11 & 11.11 & 11.11 & 11.11 & 11.11 & 11.11 & 0.00 \\
        UaIT & 0.00 & 0.00 & 0.00 & 0.00 & 0.00 & 0.00 & 6.48 & 56.06 & 35.15 & 2.31 \\
        SaySelf & 22.07 & 14.41 & 11.93 & 10.22 & 9.04 & 8.26 & 8.89 & 6.63 & 5.56 & 3.00 \\
        LoVeC & 27.33 & 3.78 & 8.17 & 0.20 & 1.37 & 0.63 & 0.22 & 1.54 & 15.80 & 40.96 \\
        \rowcolor{calibrow} CalibSFT & 11.11 & 11.11 & 11.11 & 11.11 & 5.56 & 11.11 & 11.11 & 11.11 & 11.11 & 5.56 \\
        \bottomrule
    \end{tabular}}
\end{table}

%% file: tables/target_construction_lambda.tex
\begin{table*}[t]
    \centering
    \begin{minipage}[t]{0.49\textwidth}
        \vspace{0pt}
        \centering
        \caption{AUROC and Brier score of five target construction methods on ID mathematical benchmarks. Pre-SFT metrics characterize the constructed targets, whereas Post-SFT metrics evaluate the fine-tuned models.}
        \label{tab:target_construction}
        \footnotesize
        \setlength{\tabcolsep}{3pt}
        \resizebox{\linewidth}{!}{%
        \begin{tabular}{lcccc}
        \toprule
        & \multicolumn{2}{c}{Pre-SFT targets} & \multicolumn{2}{c}{Post-SFT predictions} \\
        \cmidrule(lr){2-3} \cmidrule(lr){4-5}
        Target & AUROC $\uparrow$ & Brier $\downarrow$
        & AUROC $\uparrow$ & Brier $\downarrow$ \\
        \midrule
        CSFT & 50.00 & 31.67 & 73.46 & 19.15 \\
        UaIT & 52.05 & 33.17 & 74.72 & 24.24 \\
        SaySelf & -- & 41.33 & 57.69 & 39.38 \\
        LoVeC & 75.40 & 24.71 & 71.11 & 32.32 \\
        \rowcolor{calibrow} CalibSFT & \textbf{100.00} & \textbf{7.92}
        & \textbf{79.68} & \textbf{17.74} \\
        \bottomrule
    \end{tabular}}
    \end{minipage}\hfill
    \begin{minipage}[t]{0.45\textwidth}
        \vspace{0pt}
        \centering
        \caption{AUROC and Brier score on eight ID and eight OOD benchmarks across mixing weights. $\lambda=0.5$ achieves the highest AUROC on both splits and the lowest Brier score on ID benchmarks.}

        \label{tab:lambda_ablation}
        \footnotesize
        \setlength{\tabcolsep}{4pt}
        \resizebox{\linewidth}{!}{%
        \begin{tabular}{crrrrr}
        \toprule
        & \multicolumn{2}{c}{ID} & \multicolumn{2}{c}{OOD} \\
        \cmidrule(lr){2-3} \cmidrule(lr){4-5}
        $\lambda$ & AUROC $\uparrow$ & Brier $\downarrow$
        & AUROC $\uparrow$ & Brier $\downarrow$ \\
        \midrule
        0 & 68.10 & 32.49 & 59.20 & 38.92 \\
        0.25 & 78.02 & 18.83 & 66.88 & 29.94 \\
        \rowcolor{calibrow} 0.5 & \textbf{79.68} & \textbf{17.74} & \textbf{67.81} & 30.82 \\
        0.75 & 75.93 & 20.00 & 65.82 & \textbf{28.37} \\
        1 & 73.29 & 19.90 & 57.22 & 35.72 \\
        \bottomrule
    \end{tabular}}
    \end{minipage}
\end{table*}

%% file: tables/balance_ablation.tex
\begin{table}[t]
    \centering
    \caption{Ablation of target-balanced sampling on DeepScaleR-Eval. Balancing substantially broadens confidence support and prevents concentration at extremes, while maintaining comparable task performance (Pass@1) and calibration.}
    \label{tab:balance_ablation}
    \resizebox{\linewidth}{!}{%
    \begin{tabular}{l c ccc cccc}
        \toprule
        & \textbf{Accuracy} & \multicolumn{3}{c}{\textbf{Calibration Metrics}} & \multicolumn{4}{c}{\textbf{Confidence Diversity}} \\
        \cmidrule(lr){2-2} \cmidrule(lr){3-5} \cmidrule(lr){6-9}
        Sampling & Pass@1 $\uparrow$ & ECE $\downarrow$ & Brier $\downarrow$ & AUROC $\uparrow$
        & Support $\uparrow$ & Eff. Supp. $\uparrow$ & Top-1 $\downarrow$ & Top-3 $\downarrow$ \\
        \midrule
        Balanced   & \textbf{58.50} & \textbf{14.23} & 17.30 & 79.92 & \textbf{85.62} & \textbf{29.86} & \textbf{28.89} & \textbf{47.05} \\
        Unbalanced & 58.03          & 16.41          & \textbf{16.64} & \textbf{81.74} & 21.75          & 3.33           & 61.57          & 92.25          \\
        \bottomrule
    \end{tabular}%
    }
    \vspace{8pt}
    \includegraphics[width=0.78\linewidth]{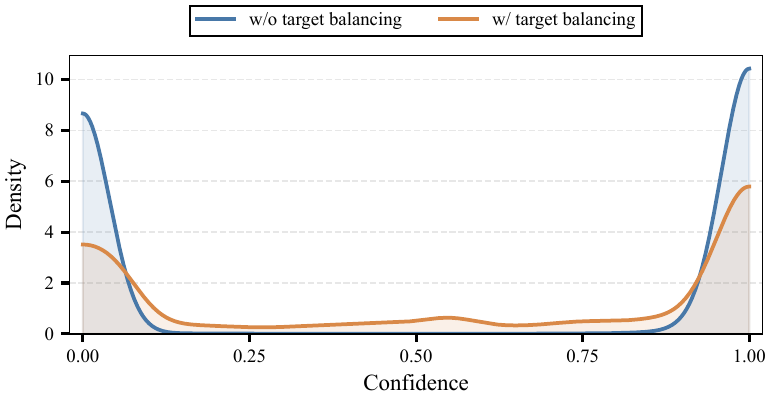}
    \captionof{figure}{Confidence distributions on DeepScaleR-Eval after CalibSFT. Target-balanced sampling prevents probability mass from collapsing to the extremes ($0$ and $1$), fostering a broader distribution.}
    \label{fig:balance_ablation_confidence}
\end{table}

%% file: tables/supervision_ablation.tex
\begin{table}[t]
    \centering
    \caption{Ablation of the CalibSFT supervision design on eight ID benchmarks.
    The variants differ in the supervised responses for each segment and in
    loss balancing.
    CalibSFT achieves the best calibration
    while keeping Pass@1 close to the best.}
    \label{tab:supervision_ablation}
    \footnotesize
    \setlength{\tabcolsep}{3pt}
    \resizebox{\linewidth}{!}{%
    \begin{tabular}{lccccc rrrr}
        \toprule
        & \multicolumn{2}{c}{Reasoning/answer}
        & \multicolumn{2}{c}{Confidence}
        & Segment & \multicolumn{4}{c}{ID performance} \\
        \cmidrule(lr){2-3} \cmidrule(lr){4-5} \cmidrule(lr){7-10}
        Variant & Correct & Incorrect & Correct & Incorrect & balanced
        & Pass@1 $\uparrow$ & AUROC $\uparrow$ & Brier $\downarrow$ & ECE $\downarrow$ \\
        \midrule
        Confidence-only & -- & -- & $\checkmark$ & $\checkmark$ & --
        & 43.56 & 74.62 & 20.97 & 17.43 \\
        Full-CE & $\checkmark$ & $\checkmark$ & $\checkmark$ & $\checkmark$ & --
        & 58.49 & 61.66 & 29.29 & 27.65 \\
        Full-Balanced & $\checkmark$ & $\checkmark$ & $\checkmark$ & $\checkmark$ & $\checkmark$
        & 58.46 & 76.67 & 19.04 & 16.87 \\
        Correct-only & $\checkmark$ & -- & $\checkmark$ & -- & $\checkmark$
        & \textbf{59.89} & 77.58 & 18.95 & 16.98 \\
        \rowcolor{calibrow} CalibSFT & $\checkmark$ & -- & $\checkmark$ & $\checkmark$ & $\checkmark$
        & 58.51 & \textbf{79.93} & \textbf{17.30} & \textbf{14.23} \\
        \bottomrule
    \end{tabular}%
    }
\end{table}

%% file: tables/confidence_estimators.tex
\begin{table}[t]
    \centering
    \caption{Confidence estimators for RLCR with and without CalibSFT initialization,
    averaged over the 8 in-distribution and 8 out-of-distribution benchmarks.
    Self-consistency aggregates the sampled responses by majority vote, so its Pass@1 column reports majority-vote accuracy.
    CalibSFT initialization improves calibration and discrimination under all three estimators.}
    \label{tab:confidence_estimators}
    \resizebox{\linewidth}{!}{%
    \begin{tabular}{llcccccccc}
        \toprule
        & & \multicolumn{4}{c}{In-distribution} & \multicolumn{4}{c}{Out-of-distribution} \\
        \cmidrule(lr){3-6} \cmidrule(lr){7-10}
        Estimator & Method & Pass@1 $\uparrow$ & AUROC $\uparrow$ & Brier $\downarrow$ & ECE $\downarrow$ & Pass@1 $\uparrow$ & AUROC $\uparrow$ & Brier $\downarrow$ & ECE $\downarrow$ \\
        \midrule
         & RLCR & 65.33 & 82.36 & 14.43 & 14.46 & 43.96 & 69.38 & 31.14 & 31.17 \\
        \rowcolor{calibrow} \cellcolor{white}\multirow{-2}{*}{Verbalized} & + CalibSFT & \textbf{67.20} & \textbf{86.32} & \textbf{12.77} & \textbf{8.59} & \textbf{45.71} & \textbf{71.31} & \textbf{23.24} & \textbf{18.20} \\
        \midrule
         & RLCR & 65.33 & 83.60 & 14.35 & 15.03 & 43.96 & 69.70 & 31.22 & 30.54 \\
        \rowcolor{calibrow} \cellcolor{white}\multirow{-2}{*}{Token probability} & + CalibSFT & \textbf{67.20} & \textbf{87.66} & \textbf{12.75} & \textbf{10.13} & \textbf{45.71} & \textbf{71.51} & \textbf{22.80} & \textbf{19.29} \\
        \midrule
         & RLCR & 70.86 & 85.36 & 13.09 & 17.62 & 44.83 & 70.83 & 30.61 & 30.53 \\
        \rowcolor{calibrow} \cellcolor{white}\multirow{-2}{*}{Self-consistency} & + CalibSFT & \textbf{72.31} & \textbf{87.64} & \textbf{11.35} & \textbf{9.67} & \textbf{47.24} & \textbf{72.51} & \textbf{22.77} & \textbf{18.19} \\
        \bottomrule
    \end{tabular}}
\end{table}

%% file: tables/training_cost.tex
\begin{table}[t]
    \centering
    \caption{Training cost comparison between CalibSFT and confidence-aware RL
    under a 200-step training budget on eight NVIDIA B20Z GPUs.
    CalibSFT takes substantially less training time and consumes only 6.9\% of the GPU-hours of RLCR.}
    \label{tab:training_cost}
    \footnotesize
    \begin{tabular}{lccc}
        \toprule
        Stage & Minutes per step & Wall-clock (h) & GPU-hours \\
        \midrule
        CalibSFT & 0.33 & 1.1 & 8.8 \\
        \midrule
        RLCR & 4.77 & 15.9 & 127.3 \\
        RLCR + CalibSFT & 4.58 & 15.3 & 122.2 \\
        \bottomrule
    \end{tabular}
\end{table}

%% file: tables/risk_control.tex
\begin{table}[t]
    \centering
    \caption{Certified coverage (\%) at fixed risk budgets for confidence-aware RL with and without CalibSFT initialization.
    Values report the mean and standard deviation on held-out questions over ten random splits ($\delta=0.05$).
    Perfect ranking serves as an oracle reference that ranks all correct responses above incorrect ones.
    CalibSFT initialization consistently increases certified coverage across all methods and risk budgets.}
    \label{tab:risk_control}
    \footnotesize
    \setlength{\tabcolsep}{5pt}
    \begin{tabular}{lccc}
        \toprule
        & \multicolumn{2}{c}{In-distribution} & Out-of-distribution \\
        \cmidrule(lr){2-3} \cmidrule(lr){4-4}
        Method & $r_{\max}=10\%$ & $r_{\max}=20\%$ & $r_{\max}=20\%$ \\
        \midrule
        RLCR & 31.4$_{\pm 0.6}$ & 86.4$_{\pm 1.1}$ & 0.0$_{\pm 0.0}$ \\
        \rowcolor{calibrow} + CalibSFT & \textbf{62.6$_{\pm 2.5}$} & \textbf{88.6$_{\pm 1.3}$} & \textbf{20.3$_{\pm 0.8}$} \\
        \midrule
        ReDoubt & 28.8$_{\pm 0.5}$ & 70.2$_{\pm 12.3}$ & 0.0$_{\pm 0.0}$ \\
        \rowcolor{calibrow} + CalibSFT & \textbf{66.1$_{\pm 0.9}$} & \textbf{88.6$_{\pm 1.2}$} & \textbf{18.0$_{\pm 2.5}$} \\
        \midrule
        CoCA & 0.0$_{\pm 0.0}$ & 82.3$_{\pm 1.4}$ & 0.0$_{\pm 0.0}$ \\
        \rowcolor{calibrow} + CalibSFT & \textbf{52.3$_{\pm 10.5}$} & \textbf{87.2$_{\pm 2.5}$} & \textbf{20.5$_{\pm 0.2}$} \\
        \midrule
        DCPO & 1.6$_{\pm 4.7}$ & 85.5$_{\pm 2.2}$ & 0.0$_{\pm 0.0}$ \\
        \rowcolor{calibrow} + CalibSFT & \textbf{33.3$_{\pm 21.4}$} & \textbf{87.3$_{\pm 1.4}$} & \textbf{22.8$_{\pm 1.5}$} \\
        \midrule
        C3RL & 0.0$_{\pm 0.0}$ & 17.7$_{\pm 35.3}$ & 0.0$_{\pm 0.0}$ \\
        \rowcolor{calibrow} + CalibSFT & \textbf{49.2$_{\pm 16.4}$} & \textbf{84.8$_{\pm 0.5}$} & \textbf{16.8$_{\pm 0.2}$} \\
        \midrule
        Perfect ranking & 82.0 & 92.3 & 54.5 \\
        \bottomrule
    \end{tabular}
\end{table}

%% file: tables/cascade.tex
\begin{table*}[t]
    \centering
    \caption{Confidence-guided model routing to Qwen3.7-Flash across two accuracy-loss tolerances $\varepsilon$, where local policies are trained on Qwen3-8B. Values are averaged over ten random splits. CalibSFT initialization substantially reduces deferral rates and increases token savings while maintaining task accuracy close to the full-deferral baseline ($82.58\%$).}
    \label{tab:cascade}
    \footnotesize
    \setlength{\tabcolsep}{5pt}
    \resizebox{\textwidth}{!}{%
    \begin{tabular}{lcccccc}
        \toprule
        & \multicolumn{3}{c}{Tolerance $\varepsilon=1\%$} & \multicolumn{3}{c}{Tolerance $\varepsilon=2\%$} \\
        \cmidrule(lr){2-4} \cmidrule(lr){5-7}
        Method & Deferred (\%) $\downarrow$ & Tokens saved (\%) $\uparrow$ & Accuracy & Deferred (\%) $\downarrow$ & Tokens saved (\%) $\uparrow$ & Accuracy \\
        \midrule
        RLCR & 65.0$_{\pm 1.1}$ & 10.1 & 82.97 & 45.9$_{\pm 1.0}$ & 23.3 & 83.33 \\
        \rowcolor{calibrow} + CalibSFT & \textbf{60.3$_{\pm 3.6}$} & \textbf{17.9} & 82.66 & \textbf{39.2$_{\pm 1.0}$} & \textbf{36.1} & 82.47 \\
        \midrule
        ReDoubt & 67.4$_{\pm 0.8}$ & 10.1 & 82.99 & 43.1$_{\pm 1.0}$ & 27.1 & 83.24 \\
        \rowcolor{calibrow} + CalibSFT & \textbf{55.3$_{\pm 2.7}$} & \textbf{23.8} & 82.81 & \textbf{36.3$_{\pm 3.7}$} & \textbf{42.1} & 82.60 \\
        \midrule
        CoCA & 98.7$_{\pm 0.1}$ & 0.6 & 82.61 & 98.7$_{\pm 0.1}$ & 0.6 & 82.61 \\
        \rowcolor{calibrow} + CalibSFT & \textbf{71.7$_{\pm 4.9}$} & \textbf{17.3} & 82.51 & \textbf{51.2$_{\pm 5.5}$} & \textbf{33.1} & 82.42 \\
        \midrule
        DCPO & 82.9$_{\pm 0.7}$ & 5.6 & 82.64 & 82.9$_{\pm 0.7}$ & 5.6 & 82.64 \\
        \rowcolor{calibrow} + CalibSFT & \textbf{77.4$_{\pm 0.3}$} & \textbf{12.2} & 82.58 & \textbf{33.3$_{\pm 3.8}$} & \textbf{43.5} & 82.80 \\
        \bottomrule
    \end{tabular}
    }
\end{table*}